\PassOptionsToPackage{hyperfootnotes=false}{hyperref}
\documentclass[a4paper,fleqn]{cas-sc}

\usepackage[authoryear]{natbib}

\providecommand{\papermode}{combined} 
\newcommand{\separatemode}{separate}
\newcommand{\combinedmode}{combined}

\ifx\papermode\separatemode
\usepackage{xr-hyper}
\makeatletter
\let\maincitation\citation
\let\mainbibcite\bibcite
\def\citation#1{}
\def\bibcite#1#2{}
\let\citation\maincitation
\let\bibcite\mainbibcite
\makeatother
\fi

\usepackage{amsmath, amsfonts, amssymb, amsthm}
\usepackage{bm}

\usepackage{graphicx}
\graphicspath{{./}{./figures/}}
\usepackage{subfigure} 
\usepackage[section]{placeins}
\usepackage{booktabs}
\usepackage{multirow}
\usepackage{longtable}
\usepackage{supertabular}
\usepackage{array}
\usepackage{tkz-euclide, pgfplots}
\pgfplotsset{compat=1.18}

\usepackage[ruled]{algorithm2e}
\SetKwInput{KwIn}{Input}
\SetKwInput{KwOut}{Output}

\usepackage{url}
\usepackage{color}
\usepackage{indentfirst}
\usepackage{enumitem}
\setlist[itemize]{noitemsep}

\ExplSyntaxOn
\RenewDocumentCommand \emailauthor { m m }
  {
    \int_gincr:N \g_ead_int
    \seq_gput_right:Nn \g_stm_ead_seq
      {
        { \nolinkurl{#1} }
        \parsename { #2 }
        \space(\eadauthor)%
      }
  }
\ExplSyntaxOff

\newtheorem{theorem}{Theorem}
\newtheorem{lemma}{Lemma}
\newtheorem{proposition}{Proposition}
\newtheorem{assumption}{Assumption}
\renewcommand{\theassumption}{\Alph{assumption}}

\newtheorem{remark}{Remark}

\newlist{subasm}{enumerate}{1}
\setlist[subasm]{
    label=(\theassumption\arabic*),  
    ref=\theassumption\arabic*,      
    itemsep=0pt,    topsep=0pt,      
    partopsep=0pt,  parsep=0pt,      
}

\makeatletter
\newcommand{\storecurrentcounter}[2]{%
  \begingroup
  \edef\@currentlabel{\number\value{#2}}%
  \label{#1}%
  \endgroup
}
\makeatother

\def\tsc#1{\csdef{#1}{\textsc{\lowercase{#1}}\xspace}}
\tsc{WGM}
\tsc{QE}
\def\orcidlink#1{}          
\renewcommand{\printorcid}{} 

\hypersetup{
    colorlinks=true,
    unicode=true
}

\begin{document}
\def\floatpagepagefraction{1}
\def\textpagefraction{.001}

\shorttitle{Byzantine-Robust Sparse Learning}

\shortauthors{Wang et al.}

\title [mode = title]{Byzantine-Robust Distributed Sparse Learning Revisited}  



%

\author[1]{Yuxuan Wang}


\ead{wangyuxuan@zju.edu.cn}



\affiliation[1]{organization={School of Mathematical Sciences},
            addressline={Zhejiang University}, 
            city={Hangzhou},
            postcode={310027}, 
            state={Zhejiang},
            country={China}}

\author[2]{Lixin Zhang}


\ead{stazlx@mail.zjgsu.edu.cn}



\affiliation[2]{organization={School of Statistics and Mathematics},
            addressline={Zhejiang Gongshang University}, 
            city={Hangzhou},
            postcode={310018}, 
            state={Zhejiang},
            country={China}}

\author[3]{Kangqiang Li}
\cormark[1]
\ead{11935023@zju.edu.cn }

\affiliation[3]{organization={Information Center},
            addressline={Hubei Provincial Tobacco Monopoly Administration}, 
            city={Wuhan},
            postcode={430030}, 
            state={Hubei},
            country={China}}
\cortext[1]{Corresponding author}



\begin{abstract}
We revisit Byzantine robust distributed estimation for high-dimensional sparse linear models. By combining local $\ell_1$-regularized robust estimation with robust aggregation at the server, the framework applies to pseudo-Huber regression, quantile regression, and sparse SVM. We show that the resulting estimators yield non-asymptotic guarantees and attain near-optimal statistical rates under mild conditions, while remaining communication-efficient. Simulations confirm strong robustness in estimation, support recovery and classification accuracy under various Byzantine attacks.
\end{abstract}



{
\emergencystretch=3em
\maketitle
}

\section{Introduction}
Distributed statistical learning and federated learning have become standard tools in large-scale and privacy-sensitive settings. At the same time, robustness has become a central concern. In the distributed setting, robustness has two distinct aspects: robustness to heavy-tailed data and to Byzantine failures. A Byzantine machine may send arbitrary messages because of hardware faults, communication errors, or adversarial attacks. Once such messages enter the central aggregation step, they can substantially distort the global update. In this regime, simple averaging and related weighted averaging are often fragile, which has made Byzantine-robust distributed learning an active area of research \citep{tolerant,stat,aggregation}.

Heavy-tailed noise brings different challenges. In high-dimensional problems, where the ambient dimension is comparable to or larger than the sample size, regularization is commonly used for estimation and variable selection. Since squared-loss methods are well known to be sensitive to outliers, heteroscedasticity, and heavy-tailed errors, robust alternatives such as quantile loss \citep{quantile}, Huber loss \citep{Huber}, and pseudo-Huber loss \citep{quanhuber} tackle this issue in different ways. Recent work on adaptive Huber regression \citep{lamm} and communication-efficient quantile regression \citep{wanglei1} has strengthened the statistical foundations of robust learning in high dimensions.

These two research directions intersect in the distributed sparse settings, where the goal is to achieve accurate support recovery and sharp error bounds under constraints of limited communication, heavy-tailed noise, and Byzantine attacks. This combination remains challenging. Existing distributed sparse learning methods, including divide-and-conquer~\citep{Distributed1} and debiased Lasso~\citep{Communication1}, are sensitive to Byzantine attacks. Byzantine-robust aggregation rules such as Krum \citep{tolerant}, coordinate-wise median, truncated mean \citep{stat}, and median-of-means \citep{mom} do provide protection against Byzantine machines, but their statistical behavior is not automatically aligned with the optimal rates associated with specific regression losses \citep{gmom}. The difficulty is compounded by the presence of $L_1$ regularization and genuinely non-smooth losses such as the quantile loss. There is therefore still a gap between robust distributed optimization and high-dimensional sparse inference.

Motivated by this gap, we study a Byzantine-robust distributed framework for sparse learning. We first recall some related work.   \citet{Communication2} proposed DANE, which integrates local empirical loss and global gradient information at each iteration. A more flexible route is provided by the Communication-Efficient Surrogate Likelihood (CSL) framework of \citet{Communication3}. Building on this line, \citet{Efficient} and \citet{Communication-efficient} developed communication-efficient iterative estimators. In particular, \citet{Communication-efficient} proposed CEASE, which combines proximal updates with communication-efficient surrogate objectives. \cite{zhou2023} adopt the penalty in this framework, but only attain a suboptimal rate, with respect to the dimension $d$, which can be reduced to $s\log d$ in this paper.
Besides, \cite{Divide} developed
a Byzantine-resilient method for the distributed sparse M-estimation problem.

Apart from regression, support vector machines (SVM) provide a classification setting where robustness is equally important. In high-dimensional sparse problems, the hinge loss is more resistant to outliers than squared loss, but distributed SVM estimation remains defenseless to Byzantine attacks. \citet{svm} studied communication-efficient distributed learning for sparse SVM, with a suboptimal convergence rate, and the performance may decline when the number of machines is large. Previous works focus primarily on statistical convergence in benign environments. The Byzantine setting remains relatively under-explored.

In this paper, we adapt the CEASE framework of \citet{Communication-efficient} to pseudo-Huber and quantile losses, with an extension to sparse SVM, under Byzantine contamination. By incorporating robust aggregation into the CSL framework, we obtain order-optimal statistical rates of the order derived in \cite{stat}. At each worker, we solve a sparse regularized problem with a robust loss; at the server, we aggregate local information by robust coordinate-wise rules to control malicious updates. This leads to an iterative procedure that is communication efficient, supports sparse recovery, and remains effective under heavy-tailed noise and Byzantine corruption. 

The rest of this paper is organized as follows. Section \ref{sec2} introduces the problem background. Section \ref{SecMain}
presents our main algorithms and provides statistical guarantees. Section \ref{Secdis} concludes the paper
and discusses possible future directions. Extensions of the results, detailed proofs, and additional numerical experiments on both synthetic and real data are provided in the supplementary material.

	\subsection*{Notations}\label{SecDef}
	For any positive integer $n$, we denote the set $\{1,2,\ldots,n\}$ by $[n]$. For two vectors $\mathbf{x},\mathbf{y} \in \mathbb{R}^{n}$, we write $\langle \mathbf{x},\mathbf{y}\rangle=\sum_{i=1}^{n}x_{i}y_{i}$ to be the inner product and for a vector $\mathbf{v} \in \mathbb{R}^{d}$, the support of $\mathbf{v}$ is defined as $\operatorname{support}(\mathbf{v})=\{i: v_{i}\neq 0\}$. For $q \in [1,\infty)$, the ${\ell}_{q}$-norm is defined as $\|\mathbf{v}\|_{q}=\left(\sum_{i=1}^{n}|v_{i}|^{q}\right)^{1/q}$ and $\|\mathbf{v}\|_{\infty}=\max_{i \in [n]}|v_{i}|$. Given a subset $S$ of index set $[d]$, the projection $\mathbf{v}_{S}$ denotes the vector $\left(\mathbf{v}_{S}\right)_{i}=(\mathbf{v})_{i}\mathbf{1}_{\{i \in S\}}, i\in[d]$. $\mathcal{C}(S, 3): = \left\{ \Delta \in \mathbb{R}^d:\|\Delta_{S^c}\|_1 \leq 3\|\Delta_S\|_1 \right\}$ is a restricted subset in $\mathbb{R}^d$. For a matrix $\mathbf{A}=(a_{ij}) \in \mathbb{R}^{n_{1}\times n_{2}}$, the ${\ell}_{\infty}$-norm of $\mathbf{A}$ is defined as $\|\mathbf{A}\|_{\infty}=\max_{i \in [n_{1}],j \in [n_{2}]}|a_{ij}|$. $\|M\|_{\mathrm{M}, \infty}=\max _{x \in \mathbb{R}^q,\|x\|_{\infty} \leq 1}\|M x\|_{\infty}=\max _{1 \leq i \leq p} \sum_{j=1}^q\left|M_{i, j}\right|$. Given two sequences $\{a_{n}\}_{n=1}^{\infty}$ and $\{b_{n}\}_{n=1}^{\infty}$, we define $a_{n}=O(b_{n})$ or $a_{n} \lesssim b_{n}$ to denote that for some positive constant $C$, $a_{n}\leq Cb_{n}$ holds for all $n$ large enough and we use the notation $a_{n} \asymp b_{n}$, if $a_{n}=O(b_{n})$ and $b_{n}=O(a_{n})$. For a differentiable function $f:\mathbb{R}^{d}\to \mathbb{R}$, we denote the gradient of $f$ by $\nabla f$.
	For any metric space $\left( \mathcal{F}, \left\| \cdot\right\| \right) $, we denote the $ \epsilon $-cover number by $ N\left(\epsilon, \mathcal{F},  \left\| \cdot\right\| \right) $.
	
\section{Problem Background}\label{sec2}
We consider a  sparse linear model with $N$ independent and identically distributed samples $\{(y_{i},\boldsymbol{x}_{i})\}_{i\in [N]}$ satisfying
\[
y=\langle\boldsymbol{\theta}^{\star}, \boldsymbol{x}\rangle +\varepsilon,
\]
where $\boldsymbol{x}\in\mathbb{R}^{d}$ is a random covariate vector, $\varepsilon$ is a possibly heavy-tailed error term, and the target parameter $\boldsymbol{\theta}^{\star}\in\mathbb{R}^{d}$ is sparse with $\|\boldsymbol{\theta}^{\star}\|_{0}\leq s\ll d$. We allow a random design with heteroscedastic errors and assume $\mathbb{E}[\varepsilon \mid \boldsymbol{x}] = 0$, so that $\mathbb{E}[y \mid \boldsymbol{x}] = \boldsymbol{x}^{\top}\boldsymbol{\theta}^{\star}$ is the identified conditional mean. In the centralized setting, a standard estimator is obtained by solving
\begin{equation}\label{eq1}
\widehat{\boldsymbol{\theta}}:=\underset{\boldsymbol{\theta} \in \mathbb{R}^{d}}{\operatorname{argmin}} \left\{ \widehat{L}\left(\boldsymbol{\theta}; \{(y_{i},\boldsymbol{x}
_{i})\}_{i\in [N]}\right)+\lambda\|\boldsymbol{\theta}\|_{1} \right\}
\end{equation}
where $\lambda\|\boldsymbol{\theta}\|_{1}$ is the penalty function enhancing the sparsity of the solution, and $\widehat{L}\left(\boldsymbol{\theta}; \{(y_{i},\boldsymbol{x}_{i})\}_{i\in [N]}\right)$ represents the empirical loss function constructed from $N$ data samples $\{(y_{i},\boldsymbol{x}_{i})\}_{i\in [N]}$.

When both the sample size $N$ and the dimension $d$ are large, solving (\ref{eq1}) on a single machine becomes computationally impractical. We therefore consider a distributed setting with $m$ workers, where the full sample is evenly split into $m$ blocks $\{(y_{i}^{(k)},\boldsymbol{x}_{i}^{(k)})\}_{i\in [n]}$, $k\in[m]$, and $n=N/m$. Worker $k$ calculates its local empirical loss $\widehat{L}_k(\boldsymbol{\theta})$ from its own subsample only. Transferring all data to one server would incur communication cost at least $O((m-1)nd)$, which is expensive in high dimensions. Our goal is therefore to design a communication-efficient distributed procedure that approximates the centralized estimation in (\ref{eq1}).

The distributed system is also contaminated by Byzantine failures. We assume that an $\alpha$ fraction of the $m$ workers, denoted as $\mathcal{B}$, may be Byzantine, sending arbitrary messages (denoted as $*$) to the server, while the remaining workers in set $\mathcal{M}=[m] \backslash \mathcal{B}$ are normal. 

\section{Main Results} \label{SecMain}
Our main theoretical results focus on sparse linear  pseudo-Huber regression and sparse SVM classification. Quantile regression is additionally studied in the supplementary material as a complementary case.

We work with the CSL framework. At round $t$, worker $k$ sends a local gradient or subgradient vector $\widehat{\boldsymbol{g}}^{(k)}(\widehat{\boldsymbol{\theta}}_{t})\in\mathbb{R}^d$. For a scalar sample $\{x_i\}_{i=1}^m$, let $x_{(1)}\leq\cdots\leq x_{(m)}$ and $b=\lfloor \beta m \rfloor$, and define
\[
\operatorname{trmean}_\beta\{x_i:i\in[m]\}=\frac{1}{m-2b}\sum_{r=b+1}^{m-b}x_{(r)}, \qquad
\operatorname{median}\{x_i:i\in[m]\}=x_{(\lceil m/2\rceil)}.
\]
Both operators are applied coordinate-wise to $\{\widehat{\boldsymbol{g}}^{(k)}(\widehat{\boldsymbol{\theta}}_{t})\}_{k=1}^m$. Thus trimmed mean removes the largest and smallest $b$ values in each coordinate, while the median keeps the middle value. If an $\alpha$ fraction of machines is Byzantine, choosing $\beta\geq\alpha$ or using the honest majority prevents corrupted messages from dominating the surrogate gradient. Algorithm~\ref{Al2} summarizes the procedure, whose per-round communication cost is $\mathcal{O}(md)$.
	\begin{algorithm}[h]
	\caption{Byzantine-robust distributed learning, adopted from \cite{zhou2023}} 
	\label{Al2}
	\KwIn{Initialize estimator $\widehat{\boldsymbol{\theta}}_{0}$,  algorithm parameter $ \beta $ (for Option \uppercase\expandafter{\romannumeral1}) and number of iteration $T$.}
	\For{$t=0,1,2,\ldots,T-1$}{
		\For{$k=1,2,3,\ldots,m$}{
			Each machine $ k $: calculate
			\[
			\widehat{\boldsymbol{g}}^{(k)}\left(\widehat{\boldsymbol{\theta}}_{t}\right) \leftarrow \begin{cases}\nabla\widehat{L}_k(\widehat{\boldsymbol{\theta}}_{t}) & \text { normal machines } \\ * & \text { byzantine machines }\end{cases}
			\]
			and send to the 1st machine.
		}
		The 1st machine: evaluates aggregate gradients through robust methods: \[\nabla \widetilde{\mathcal{L}}(\widehat{\boldsymbol{\theta}}_{t})\leftarrow \begin{cases}\text{trmean}_{\beta}\left\{\widehat{\boldsymbol{g}}^{(k)}\left(\widehat{\boldsymbol{\theta}}_{t}\right)\right\} & \text { Option \uppercase\expandafter{\romannumeral1} } \\ \text{median}_{k \in [m]}\left\{\widehat{\boldsymbol{g}}^{(k)}\left(\widehat{\boldsymbol{\theta}}_{t}\right)\right\} & \text { Option \uppercase\expandafter{\romannumeral2}}\end{cases}
		\]
		solves the following regularized optimization problem:
		\begin{equation}\label{eq2}
			\widehat{\boldsymbol{\theta}}_{t+1}:=\underset{\boldsymbol{\theta} \in \mathbb{R}^{d}}{\operatorname{argmin}} \left\{\widehat{L}_1(\boldsymbol{\theta})-
			\left\langle \nabla \widehat{L}_1(\widehat{\boldsymbol{\theta}}_{t})-\nabla \widetilde{\mathcal{L}}(\widehat{\boldsymbol{\theta}}_{t}),\boldsymbol{\theta}\right\rangle+\lambda_{t+1}\|\boldsymbol{\theta}\|_{1}
			\right\}\end{equation}\label{def}
		and then broadcasts $\widehat{\boldsymbol{\theta}}_{t+1}$ to the other $ m - 1 $ machines.
	}\KwOut{$\widehat{\boldsymbol{\theta}}_{T}$.}
\end{algorithm}
	
\subsection{Pseudo-Huber Regression}

We begin with the Huber-type robust regression. This serves as a natural starting point, as the pseudo-Huber loss preserves the robustness of Huber-type methods while remaining smooth enough for optimization and analysis. We firstly introduce the empirical objective:
\[
\widehat{L}(\boldsymbol{\theta}):=(1/N) \sum_{i=1}^N \ell_a\left(y_i-\boldsymbol{x}_i^{\top} \boldsymbol{\theta}\right),
\]
where $\ell_a(x)=2 a^{-2}\left(\sqrt{1+a^2 x^2}-1\right)$, and the associated sparse regularized estimator is
\[
\widehat{\boldsymbol{\theta}}\in\arg\min\limits_{\boldsymbol{\theta}\in \mathbb{R}^d} \left\{\widehat{L}(\boldsymbol{\theta})+\lambda\left\|\boldsymbol{\theta}\right\|_1\right\}.
\]
\begin{assumption}\label{assumption1}
\begin{subasm}

		\item $\boldsymbol{\theta}^{\star}$ is sparse and the support of $\boldsymbol{\theta}^{\star}$ is denoted by $S$. We have $|S| \leq s$ for some $s \leq \min \{d, n\}$. \label{a1}
		\item We already have the initial estimator $\widehat{\boldsymbol{\theta}}_{0}$ sparse, that is, for some universal $C > 0$, $\|\widehat{\boldsymbol{\theta}}_{0}\|_0\leq Cs$. And $\|\widehat{\boldsymbol{\theta}}_{0}-\boldsymbol{\theta}^{\star}\|_{1} \leq Cs\sqrt{\log d/n}.$ \label{a2}
		\item  $\boldsymbol{x}$ is a sub-Gaussian random vector with coordinates of mean zero and variance 1. \label{a3}
		\item   $0<c_{1}\leq \lambda_{\min}(\boldsymbol{\Sigma})\leq \lambda_{\max}(\boldsymbol{\Sigma})\leq c_{2}<\infty$, $\|\left(\boldsymbol{\Sigma}_{S S}\right)^{-1}\|_{\mathrm{M}, \infty} \leq C$, where $\boldsymbol{\Sigma}:=\mathbb{E} [\boldsymbol{x}\boldsymbol{x}^{\top}]$ and $\boldsymbol{\Sigma}_{S S}:=\mathbb{E} [\boldsymbol{x}_S\boldsymbol{x}_S^{\top}]$.
		$ \mathbb{E}[\mathbb{E}[|\varepsilon_i|^{3}\mid \boldsymbol{x}_i]]<\infty $. \label{a4}
\end{subasm}
\end{assumption}
\begin{remark}
(\ref{a1}) is the fundamental assumption for our theoretical analysis and typically appears in penalized variable selection. A lot of high-dimensional vectors naturally appear as (approximately) sparse vectors, for instance, genetic profiles where only a few genes are truly associated with a trait, or text feature vectors where most word frequencies are zero. Thus, we believe (\ref{a1}) is a reasonable assumption. Assumptions (\ref{a3}) and (\ref{a4}) are essentially those from \cite{assum} and further modified by \cite{supp}. Specifically, (\ref{a4}) can be regarded as a restricted eigenvalue condition.

To achieve $\|\widehat{\boldsymbol{\theta}}_{t}\|_0\leq Cs$, we have several possibilities. One approach is to project $ \widehat{\boldsymbol{\theta}}_{t}$ to the space of vectors of sparsity $s$ to obtain \( \widetilde{\boldsymbol{\theta}}_{t} \), see, for example, \cite{supp}, since with probability at least $ 1 - c_1 \exp(-c_2 n) - c_3/d^2 $, where $c_1, c_2, c_3 > 0$ are suitable constants, the initial estimator $\widehat{\boldsymbol{\theta}}_0$ satisfies $ \operatorname{sign}(\widehat{\boldsymbol{\theta}}_0) = \operatorname{sign}\left( \boldsymbol{\theta}^* \right)$. On that event, we take $\widetilde{\boldsymbol{\theta}}_t=(\widehat{\boldsymbol{\theta}}_t)_{\widehat{S}_0}$, where $\widehat{S}_0$ is the support of $\widehat{\boldsymbol{\theta}}_0$, and we have
\[\|\widetilde{\boldsymbol{\theta}}_t-\boldsymbol{\theta}^{\star}\|_1=\|(\widehat{\boldsymbol{\theta}}_t-\boldsymbol{\theta}^{\star})_S\|_1\leq \|\widehat{\boldsymbol{\theta}}_t-\boldsymbol{\theta}^{\star}\|_1.\]
\end{remark}

\begin{theorem}\label{theorem1}
	Suppose that Assumption~\ref{assumption1} holds. Consider Option I in Algorithm~\ref{Al2}, and choose the robustification parameter $\beta\geq\alpha$, 
	and the regularization parameter $ \lambda_{t} $ of order
$		\lambda_{t}\asymp \sqrt{(\log d)/N}+\beta\sqrt{(\log d)/n}.$
Then, with probability at least $1-Tmd^{-C}$,
\[
\left\|\widehat{\boldsymbol{\theta}}_{T}-\boldsymbol{\theta}^{*}\right\|_{2} \lesssim \sqrt{\frac{s\log  d}{N}}+\beta \sqrt{\frac{s\log  d}{n}}+\sqrt{\frac{s\log d}{n}}\left( \frac{s\log d}{\sqrt{n}}\right) ^{T} ,
\]
and $\|\widehat{\boldsymbol{\theta}}_{T}-\boldsymbol{\theta}^{*}\|_{1} \lesssim s\sqrt{\log d/N}+\beta\, s\sqrt{\log d/n}+s\sqrt{\log d/n}\,(s\log d/\sqrt{n})^{T}$,
where $C>0$ is a universal constant.
\end{theorem}
	\begin{theorem}\label{theorem2}
	Suppose that Assumption~\ref{assumption1} holds. Consider Option \uppercase\expandafter{\romannumeral2} in Algorithm~\ref{Al2}, and choose the regularization parameter $ \lambda_{t} $ of order:
$		\lambda_{t}\asymp \sqrt{(\log d)/N}+\alpha/\sqrt{n}+\sqrt{(\log d)/n}\left( \sqrt{(s\log d)/n} \right) ^{t}.$
	Then, with probability at least $ 1-Tmd^{-C} $,
	\[
	\left\|\widehat{\boldsymbol{\theta}}_{T}-\boldsymbol{\theta}^{*}\right\|_{2} \lesssim \sqrt{\frac{s\log  d}{N}}+\alpha\sqrt{\frac{s}{n} }+\frac{\sqrt{s}}{n}+\sqrt{\frac{s\log d}{n}}\left( \sqrt{\frac{s\log d}{n} } \right) ^{T},
	\]
	where $ C>0 $ is a universal constant.
\end{theorem}

\begin{remark}
Note that when $ T $ is sufficiently large (as long as $\log(m)\lesssim T$), we give the byzantine robust error rate in sparse setting, matching the order-optimal statistical rates $\widetilde{\mathcal{O}}\left(\alpha/\sqrt{n}+1/\sqrt{N}\right)$, which is derived by \citet{stat} using robust distributed gradient descent (GD) algorithms. Also, we make it communicational efficient: with $\mathcal{O}(md)$ communication per round, with as few rounds as possible.
\end{remark}

\subsection{Support Vector Machine}
Now consider a two-class classification problem using i.i.d. data $\{(y_i, \boldsymbol{x}_i)\}_{i=1}^N$, where $y_i \in \{-1, 1\}$ denotes the class label and $\boldsymbol{x}_i \in \mathbb{R}^d$ represents the feature vector. We adopt the regularization of SVMs, which is formulated as:
\begin{equation}\label{em}
	(\widehat{b}, \widehat{\boldsymbol{\theta}}) = \arg\min_{b, \boldsymbol{\theta}} \frac{1}{N}\sum_{i=1}^N \left( 1 - y_i \left( \boldsymbol{x}_i^\top \boldsymbol{\theta} + b \right) \right)_+ + \lambda \|\boldsymbol{\theta}\|_1.
\end{equation}
Here:   $\left(1 - x\right)_+ = \max\{1 - x, 0\}$ denotes the hinge loss, which encourages $y_i$ and $\boldsymbol{x}_i^\top \boldsymbol{\theta} +b$ to have the same sign (i.e., correct classification).
 $(b, \boldsymbol{\theta}) \in \mathbb{R} \times \mathbb{R}^d$ are the parameters to be estimated, where  $b$ is the intercept, wihch we don't take  into account in our case without loss of  generality.  Assume $b=0$, then $\widehat{L}(\boldsymbol{\theta}) := (1/N)\sum_{i=1}^N \left( 1 - y_i \left( \boldsymbol{x}_i^\top \boldsymbol{\theta} \right) \right)_+$ is the empirical loss function.
The optimization problem for the true data distribution of (\ref{em}) is given by:
\[
(b^\star, \boldsymbol{\theta}^\star) = \arg\min_{b, \boldsymbol{\theta}} L(b, \boldsymbol{\theta}) := \mathbb{E}\left[ \left( 1 - y \left( \boldsymbol{x}^\top \boldsymbol{\theta} + b \right) \right)_+ \right].
\]
We assume that the optimal population parameters $(b^\star, \boldsymbol{\theta}^\star)$ exist and are uniquely determined.
\begin{assumption}\label{assumption3}
	The conditional densities of $\boldsymbol{x}^\top \boldsymbol{\theta}^\star$ given $y = 1$ and $y = -1$ are denoted as $f$ and $g$, respectively.
	\begin{subasm}
		\item $f$ is uniformly bounded away from $0$ and $\infty$ in a neighborhood of $1$, and $g$ is uniformly bounded away from $0$ and $\infty$ in a neighborhood of $-1$.\label{b1}
		\item  Denote $\boldsymbol{X} = (\boldsymbol{x}_1, \boldsymbol{x}_2, \dots, \boldsymbol{x}_n)^T$ to be the feature design matrix and define restricted eigenvalues as follows,
\[
\lambda_{max} = \max_{\boldsymbol{\delta} \in \mathbb{R}^{d}: \|\boldsymbol{\delta}\|_0 \leq Cs} \frac{\boldsymbol{\delta}^T \boldsymbol{X}^T \boldsymbol{X} \boldsymbol{\delta}}{n\|\boldsymbol{\delta}\|_2^2},
\quad
\lambda_{min}(H(\boldsymbol{\theta}^\star); s) = \min_{\boldsymbol{\delta} \in \mathcal{C}(S,3)} \frac{\boldsymbol{\delta}^T H(\boldsymbol{\theta}^\star) \boldsymbol{\delta}}{\|\boldsymbol{\delta}\|_2^2},
\]
where  $S \subset [d]$ , $|S| \leq s$, and $H(\cdot)$ is the  Hessian matrix of $L(\boldsymbol{\theta})$. Then, $\lambda_{max}$ and $\lambda_{min}$ are bounded away from zero.\label{b2}
	\end{subasm}
\end{assumption}

\begin{theorem}\label{theorem4}
	Suppose Assumptions (\ref{a1})(\ref{a2})(\ref{a3}) and \ref{assumption3} hold. Consider one round of Algorithm~\ref{Al2} applied to the hinge loss with robust aggregation Option~\uppercase\expandafter{\romannumeral1}, and let $\widehat{\boldsymbol{\theta}}$ denote the resulting distributed estimator. Choose the robustification parameter $\beta>\alpha$, and let the regularization parameter satisfy
	\[
	\lambda \asymp \sqrt{\frac{\log d}{N}}+\beta\sqrt{\frac{\log d}{n}}+\frac{s\log d}{n^{3/4}},
	\]
	with $\lambda \ge 2\|\nabla \widehat{\mathcal{L}}(\boldsymbol{\theta}^{\star})\|_\infty$. Then, with probability at least $1-cmd^{-C}-cn^{-C}$,
	\[
	\left\|\widehat{\boldsymbol{\theta}}-\boldsymbol{\theta}^{\star}\right\|_{2} \lesssim \sqrt{\frac{s\log d}{N}}+\beta\sqrt{\frac{s\log d}{n}}+\frac{s^{3/2}\log d}{n^{3/4}}.
	\]
\end{theorem}

\begin{remark}\label{remark1}
	Although Theorem~\ref{theorem4} is stated for one round, the same argument for multi-round estimator yields the following rate as long as $n \geq \max\{C s^2\log d, N^{2/3}\}$ and $T \geq C \log(m)$:
	\[
	\left\|\widetilde{\boldsymbol{\theta}}_{T}-\boldsymbol{\theta}^{*}\right\|_{2} = O_{P}\!\left(\sqrt{(s \log d)/N}\right),
	\]
	which means that the multi-round distributed estimator attains the same rate as the full-sample estimator in \eqref{eq1} while using $\mathcal{O}(md)$ communication per round.
\end{remark}


\section{Discussion}\label{Secdis}

In this paper, we revisited two Byzantine-robust distributed surrogate function algorithms for sparse linear regression and extended them to SVM, by integrating coordinate-wise median and trimmed mean operations into the learning process. We investigate how to design communication-efficient estimators that achieve near-optimal statistical rates while supporting non-smooth losses and sparse regularization, under the existance of Byzantine attacks and heavy-tailed noise. The key is to combine a communication-efficient iterative method with robust aggregation, which controls anomalous updates from Byzantine workers and heavy-tailed samples before they enter the error recursion.

\citet{Communication-efficient} proposed the CEASE framework, which was originally designed for smooth loss functions. We extend it to two non-smooth settings: hinge loss and quantile regression. 
\cite{stat} established the first statistical lower bound for distributed estimation under Byzantine attacks.
\cite{zhou2023} introduced a penalty term but only obtained a rate of order $O(\sqrt{d})$ with respect to the dimension $d$.
In contrast, we achieve a near-optimal rate of order $O(\sqrt{s\log d/N}+\alpha\sqrt{s/n})$ in sparse settings,
using the technical approach based on surrogate likelihood.

The present results still have limitations. Our analysis depends on structural assumptions like restricted eigenvalue conditions, strong convexity, and  bounded conditional densities. These conditions are useful for theory, but they may be restrictive in applications. Another practical issue is that the Byzantine fraction $\alpha$ is usually unknown, so trimming levels cannot be fixed in advance without some form of adaptation. Future work could therefore focus on weaker and more verifiable assumptions, aggregation and tuning methods that adapt to unknown $\alpha$ and better use of heterogeneity in non-IID data, under more challenging attack models.

\section*{Acknowledgments}
This work is supported by  grants from National Key R\&D Program of China (No. 2024YFA1013502), NSF of China (Grant No. U23A2064) and the Summit Advancement Disciplines of Zhejiang Province (Zhejiang Gongshang University - Statistics).
\storecurrentcounter{main:last-section}{section}
\storecurrentcounter{main:last-equation}{equation}
\storecurrentcounter{main:last-theorem}{theorem}
\storecurrentcounter{main:last-lemma}{lemma}
\storecurrentcounter{main:last-proposition}{proposition}
\storecurrentcounter{main:last-assumption}{assumption}
\storecurrentcounter{main:last-corollary}{corollary}
\storecurrentcounter{main:last-definition}{definition}
\storecurrentcounter{main:last-remark}{remark}
\storecurrentcounter{main:last-example}{example}
\storecurrentcounter{main:last-figure}{figure}
\storecurrentcounter{main:last-table}{table}
\storecurrentcounter{main:last-algorithm}{algocf}

\ifx\papermode\combinedmode
\clearpage
\appendix
\providecommand{\papermode}{combined}
\providecommand{\separatemode}{separate}
\providecommand{\combinedmode}{combined}

\ifx\papermode\combinedmode
\newcommand{\suppenddocument}{}
\else
\documentclass[11pt]{article}

\usepackage[margin=1in]{geometry}
\usepackage[authoryear]{natbib}
\usepackage{amsmath, amsfonts, amssymb, amsthm}
\usepackage{bm}
\usepackage{graphicx}
\graphicspath{{./}{./figures/}}
\usepackage{subfigure}
\usepackage[section]{placeins}
\usepackage{booktabs}
\usepackage{threeparttable}
\usepackage{multirow}
\usepackage{array}
\usepackage{tkz-euclide, pgfplots}
\pgfplotsset{compat=1.18}
\usepackage[ruled]{algorithm2e}
\SetKwInput{KwIn}{Input}
\SetKwInput{KwOut}{Output}
\usepackage{url}
\usepackage{color}
\usepackage{indentfirst}
\usepackage{enumitem}
\setlist[itemize]{noitemsep}
\usepackage{refcount}
\usepackage[colorlinks=true, unicode=true]{hyperref}
\usepackage{xr-hyper}
\makeatletter
\newcommand{\externaldocumentnobib}[1]{%
  \begingroup
  \let\bibcite\@gobbletwo
  \let\citation\@gobble
  \let\bibdata\@gobble
  \let\bibstyle\@gobble
  \externaldocument{#1}%
  \endgroup
}
\makeatother
\externaldocumentnobib{main}

\newtheorem{theorem}{Theorem}
\newtheorem{lemma}{Lemma}
\newtheorem{proposition}{Proposition}
\newtheorem{assumption}{Assumption}
\renewcommand{\theassumption}{\Alph{assumption}}
\newtheorem{corollary}{Corollary}
\newtheorem{definition}{Definition}
\newtheorem{remark}{Remark}
\newtheorem{example}{Example}

\newlist{subasm}{enumerate}{1}
\setlist[subasm]{
    label=(\theassumption\arabic*),  
    ref=\theassumption\arabic*,      
    itemsep=0pt,    topsep=0pt,      
    partopsep=0pt,  parsep=0pt,      
}

\title{Supplementary Material for\\Byzantine-Robust Distributed Sparse Learning Revisited}
\author{}
\date{}
\newcommand{\suppenddocument}{\end{document}}

\makeatletter
\newcommand{\continuecounterfrommain}[2]{%
  \IfRefUndefinedExpandable{#1}{}{%
    \setcounter{#2}{\numexpr\getrefnumber{#1}\relax}%
  }%
}
\makeatother

\begin{document}
\maketitle
\fi

\ifx\papermode\separatemode
\continuecounterfrommain{main:last-section}{section}
\continuecounterfrommain{main:last-equation}{equation}
\continuecounterfrommain{main:last-theorem}{theorem}
\continuecounterfrommain{main:last-lemma}{lemma}
\continuecounterfrommain{main:last-proposition}{proposition}
\continuecounterfrommain{main:last-assumption}{assumption}
\continuecounterfrommain{main:last-corollary}{corollary}
\continuecounterfrommain{main:last-definition}{definition}
\continuecounterfrommain{main:last-remark}{remark}
\continuecounterfrommain{main:last-example}{example}
\continuecounterfrommain{main:last-figure}{figure}
\continuecounterfrommain{main:last-table}{table}
\continuecounterfrommain{main:last-algorithm}{algocf}
\fi



\appendix
\noindent\textbf{Organization of the Supplement.}
This supplement collects details omitted from the main paper due to space constraints. Section~\ref{sec:supp-defs} lists key definitions and notation that are abbreviated in the main text and repeatedly used in later technical arguments. Section~\ref{sec:supp-quantile} covers sparse quantile regression as an additional model. Section~\ref{sec:supp-proofs} provides proofs of the main results. Section~\ref{Secexper} reports additional numerical experiments.

\section{Definitions and Additional Notation}\label{sec:supp-defs}

Recall that, we use $m$ workers indexed by $[m]=\{1,\dots,m\}$ and total sample size $N=mn$. An $\alpha$ fraction of workers may be Byzantine: $\mathcal{B}\subseteq[m]$ denotes the Byzantine set with $|\mathcal{B}|=\alpha m$, and $\mathcal{M}=[m]\setminus\mathcal{B}$ denotes the set of honest workers.

At communication round $t$, worker $k$ sends a vector message
$\widehat{\boldsymbol{g}}^{(k)}(\widehat{\boldsymbol{\theta}}_t)\in\mathbb{R}^d$
(a gradient or subgradient evaluated at $\widehat{\boldsymbol{\theta}}_t$). When $k\in \mathcal{M}$, $\widehat{\boldsymbol{g}}^{(k)}(\cdot)=\nabla \widehat{L}_k(\cdot)$, and $\widehat{\boldsymbol{g}}^{(k)}$ could be any vector in $\mathbb{R}^d$ otherwise.  Its $j$th coordinate is denoted by
\[
\widehat g_j^{(k)}=\left[\widehat{\boldsymbol{g}}^{(k)}(\widehat{\boldsymbol{\theta}}_t)\right]_j,\qquad k\in[m],
\]
and when the honest workers are identically distributed, we write $\boldsymbol{g}(\boldsymbol{\theta}):=\mathbb{E}\!\left[\nabla \widehat{L}_{k}(\boldsymbol{\theta})\right]$, and $g_j(\boldsymbol{\theta}):=\mathbb{E}\!\left[\nabla \widehat{L}_{k}(\boldsymbol{\theta})\right]_j$. 

Let $\widehat g_{j,(1)}\le \cdots \le \widehat g_{j,(m)}$ be the order statistics of $\{\widehat g_j^{(k)}\}_{k=1}^m$.
and define the trimmed/untrimmed index sets as follows. Let $b=\beta m$ and assume $b$ is an integer. Let $\mathcal{T}_j(t)\subseteq[m]$ denote the set of trimmed workers for coordinate $j$ at round $t$, obtained by removing $b$ indices corresponding to the smallest values of $\{\widehat g_j^{(k)}\}_{k=1}^m$ and $b$ indices corresponding to the largest values of $\{\widehat g_j^{(k)}\}_{k=1}^m$ (ties are broken arbitrarily). Let
\[
\mathcal{U}_j(t)=[m]\setminus \mathcal{T}_j(t)
\]
be the set of untrimmed workers, so that $|\mathcal{U}_j(t)|=(1-2\beta)m$.
Then the trimmed mean at coordinate $j$ is
\[
\operatorname{trmean}_\beta\{\widehat g_j^{(k)}:k\in[m]\}
=\frac{1}{(1-2\beta)m}\sum_{k\in \mathcal{U}_j(t)} \widehat g_j^{(k)},
\]
and the coordinate-wise median is
\[
\operatorname{median}\{\widehat g_j^{(k)}:k\in[m]\}=\widehat g_{j,(\lceil m/2\rceil)}.
\]
Applying these operators to vectors means applying them coordinate-wise across $\{\widehat{\boldsymbol{g}}^{(k)}(\widehat{\boldsymbol{\theta}}_t)\}_{k=1}^m$. The resulting aggregated vector at round $t$ is denoted by $\nabla \widetilde{\mathcal{L}}(\widehat{\boldsymbol{\theta}}_t)$ in Algorithm~\ref{Al2}. 
Besides, we will write the $\tau$-quantile of the received messages as
\[Q_\tau\{\widehat g_j^{(k)}:k\in[m]\}=\widehat g_{j,(\tau m)}.
\]

In the proofs, we denote the surrogate objective in \eqref{eq2} by $\widehat{\mathcal{L}}(\boldsymbol{\theta},\widehat{\boldsymbol{\theta}}_t)$ and abbreviate $\widehat{\mathcal{L}}(\boldsymbol{\theta},\widehat{\boldsymbol{\theta}}_0)$ as $\widehat{\mathcal{L}}(\boldsymbol{\theta})$ when the context is clear, and $d_t:=\|\boldsymbol{\theta}^\star-\widehat{\boldsymbol{\theta}}_t\|_2$ denotes the $\ell_2$ estimation error at round $t$. For any general loss function $\ell$, we write $\tilde{\ell}(\cdot):=(\ell(\cdot))'$, as the gradient or subgradient of the loss function with respect to the parameter $\boldsymbol{\theta}$.

\section{Quantile Regression}\label{sec:supp-quantile}

We next turn to sparse quantile regression. Unlike mean regression, quantile regression focus on the conditional distribution at the level $\tau$ and is therefore robust to data with asymmetry, heteroscedasticity and heavy-tailed distributions. We use the quantile loss function $\rho_\tau(x)=x(\tau-\mathbf{1}_{\{x\leq 0\}})$. The true value $ \boldsymbol{\theta}^{*} $ can be considered as the minimizer of $ E\left[ l(\boldsymbol{\theta})\right]  $.
We obtain our local estimator through solving the following problem:
\[
\widehat{\boldsymbol{\theta}}\in\arg\min\limits_{\boldsymbol{\theta}\in \mathbb{R}^d} \left\{\widehat{L}_1(\boldsymbol{\theta})+\lambda\left\|\boldsymbol{\theta}\right\|_1\right\}, \; \text{where}\;\widehat{L}_1(\boldsymbol{\theta})=(1/n) \sum_{i=1}^n \rho_\tau\left(y_i-\boldsymbol{x}_i^{\top} \boldsymbol{\theta}\right),
\]

Meanwhile, quantile regression introduces a non-smooth loss. The lack of differentiability of the quantile loss function means that arguments based on local smoothness cannot be directly applied, and requires specialized assumptions. We introduce the assumptions from \cite{wanglei1} as follows:
\begin{assumption}\label{assumption2}
	Let $f(\cdot\mid\boldsymbol{x})$ and $F(\cdot\mid\boldsymbol{x})$ be the conditional density and conditional cumulative distribution function of $ y $ given $ \boldsymbol{x} $, respectively.
	\begin{subasm}
		\item $f(y\mid\boldsymbol{x})$ is bounded and continuously differentiable in $y$ for all $\boldsymbol{x}$ in the support of $\boldsymbol{x}$, and $f^{\prime}(y \mid\boldsymbol{x})$ (derivative with respect to $y$) is bounded in absolute value by a constant, uniformly in $y$ and $\boldsymbol{x}$. $ f\left(\boldsymbol{x}^{\top} \boldsymbol{\theta}^{\star} \mid \boldsymbol{x}\right)$ is uniformly bounded away from zero.\label{c1}
		\item 
	Let
	$$\kappa_J^2:=\inf _{\delta \in \mathcal{C}(S,3), \delta \neq 0} \frac{\delta^{\mathrm{T}} E[\boldsymbol{x}\boldsymbol{x}^{\mathrm{T}}]\delta} {\|\delta_{S \cup \bar{S}(\delta, J)}\|^2},$$
	where $S$ is the support of $\boldsymbol{\theta}^\star$ (indices of its nonzero components) and $\bar{S}(\boldsymbol{\delta}, J) \subset\{1, \ldots, d\} \backslash S$ is the support of the $J$ largest in absolute value components of $\boldsymbol{\delta}$ outside $S$. Also, let
	$$
	q:=\inf _{\delta \in \mathcal{C}(S,3), \delta \neq 0} \frac{\left(E\left|\boldsymbol{x}_i^{\mathrm{T}} \boldsymbol{\delta}\right|^2\right)^{3 / 2}}{E\left|\boldsymbol{x}_i^{\mathrm{T}} \boldsymbol{\delta}\right|^3}.
	$$
	These quantities are the same as that defined in \cite{quansparse}.
Then,	$ \kappa_s $ and $ q $ are bounded away from zero.\label{c2}
	\end{subasm}

\end{assumption}

\begin{proposition}\label{theorem3}
	Suppose Assumptions (\ref{a1})(\ref{a2})(\ref{a3})  and \ref{assumption2} hold. Consider one round of Algorithm~\ref{Al2} with robust aggregation Option~\uppercase\expandafter{\romannumeral1}, and let $\widehat{\boldsymbol{\theta}}$ denote the resulting distributed estimator. Choose the robustification parameter $\beta>\alpha$, and let the regularization parameter satisfy
	\begin{equation}\label{eq5}
		\lambda \asymp \sqrt{(\log d)/N}+\beta\sqrt{(\log d)/n}+(s\log d)/n^{3/4},
	\end{equation}
	with $\lambda \ge 2\|\nabla \widehat{\mathcal{L}}(\boldsymbol{\theta}^{\star})\|_{\infty}$. Then, with probability at least $1-cmd^{-C}-cn^{-C}$,
	\[
	\left\|\widehat{\boldsymbol{\theta}}-\boldsymbol{\theta}^{*}\right\|_{2}\lesssim \sqrt{\frac{(s\log d)}{N} }+\beta\sqrt{\frac{s\log d}{n}}+\frac{s^{3/2}\log d}{n^{3/4}} .
	\]
\end{proposition}
Proposition~\ref{theorem3} is the quantile regression counterpart included for completeness.

\section{Proofs of Main Results}\label{sec:supp-proofs}
\noindent\textbf{Proof Roadmap.} The proofs of Theorems~\ref{theorem1} and \ref{theorem2} share the same recursive structure: Proposition~\ref{thm0} gives the one-step error bound and Proposition~\ref{spa} provides the sparsity needed to close the recursion, while the only aggregation-specific difference is that Theorem~\ref{theorem1} uses the trimmed-mean gradient control in Proposition~\ref{star} whereas Theorem~\ref{theorem2} uses the median-based bound in Proposition~\ref{star'}. The key technique to handle Byzantine corruption is to apply Lemma~\ref {trmean} and the equation adopted from \cite {Divide}.

The proof of Theorem~\ref{theorem4} and Proposition~\ref{theorem3} uses the same proof pattern: a generalized one-step estimation bound obtained by the regularity argument, with a robust gradient control bound for the norm of the gradient.

\subsection{Proof for Pseudo-Huber Regression}

To start, straightforward differentiation gives
$$
\ell_a'(x) = \frac{2x}{\sqrt{1 + a^2x^2}}
\quad \text{so\ that} \quad
\left| \ell_a'(x) \right| \leq \frac{2|x|}{\sqrt{a^2x^2}} = 2a^{-1},
$$
and
$$
\ell_a''(x) = \frac{2a^{-3}}{\left(a^{-2} + x^2\right)^{3/2}}
\quad \text{so\ that} \quad
0 < \ell_a''(x) \leq \frac{2a^{-3}}{\left(a^{-2}\right)^{3/2}} = 2.
$$
In particular $\ell_a$ is strictly convex. Also note that $\lim_{a \to 0} \ell_a(x) = x^2$ for all $x \in \mathbb{R}$.

\noindent\textbf{Proof Sketch.}
We establish a recursive estimation error bound that relates the estimation error $\|\widehat{\boldsymbol{\theta}}_{t+1} - \boldsymbol{\theta}^*\|_2$ at iteration $t+1$ to the estimation error $\|\widehat{\boldsymbol{\theta}}_t - \boldsymbol{\theta}^*\|_2$ at the previous iteration, via the estimation of $\lambda_{t+1}$. We first analyze how the estimation error bound decreases
after one round of communication.

\begin{proposition}\label{thm0}
	Suppose Assumption~\ref{assumption1} holds. If we calculate $\widehat{\boldsymbol{\theta}} $ using Algorithm~\ref{Al2}, with parameters $\lambda\geq 2\left\| \nabla \widehat{\mathcal{L}}\left(\boldsymbol{\theta}^{\star},\widehat{\boldsymbol{\theta}}_{0}\right)\right\|_{\infty}$, then $\|\widehat{\boldsymbol{\theta}}_{1} - \boldsymbol{\theta}^*\|_1 \leq 4\sqrt{s}\|\widehat{\boldsymbol{\theta}}_{1} - \boldsymbol{\theta}^*\|_2$, and with probability at least $1-e^{-cn}$, we have
	$$
	\|\widehat{\boldsymbol{\theta}}_{1} - \boldsymbol{\theta}^*\|_2 \leq C\sqrt{s}\lambda.
	$$
\end{proposition}

The following proposition bounds the $\ell_\infty$-norm of the gradient of the surrogate loss function at the true parameter point $\boldsymbol{\theta}^*$.

\begin{proposition}\label{star}
	If the corresponding assumptions hold,
	then with probability at least $1-tmd^{-C}$, choosing Option \uppercase\expandafter{\romannumeral1} in Algorithm~\ref{Al2}, we have 
	$$\left\| \nabla \widehat{\mathcal{L}}\left(\boldsymbol{\theta}^{\star},\widehat{\boldsymbol{\theta}}_{t}\right)\right\|_{\infty}\lesssim \sqrt{\frac{\log d}{N}}+\beta \sqrt{\frac{\log d}{n}}+\sqrt{\frac{\log d}{n}}\left( \frac{s\log d}{\sqrt{n}}\right) ^{t}.$$
\end{proposition}

\begin{proposition}\label{star'}
	If the corresponding assumptions hold, and the fraction $\alpha$ of Byzantine machines satisfies $$
	\alpha + \sqrt{\frac{\log d}{m(1-\alpha)}} + 0.4748 \frac{S}{\sqrt{n}} \leq \frac{1}{2}-\epsilon,$$
	where $S$ is the supremum of the absolute skewness of the gradient of the loss function, then with probability at least $1-tmd^{-C}$, choosing Option \uppercase\expandafter{\romannumeral2} in Algorithm~\ref{Al2}, we have 
	$$\left\| \nabla \widehat{\mathcal{L}}\left(\boldsymbol{\theta}^{\star},\widehat{\boldsymbol{\theta}}_{t}\right)\right\|_{\infty}\lesssim \sqrt{\frac{\log  d}{N}}+\frac{\alpha}{\sqrt{n} }+\frac{1}{n}+\sqrt{\frac{\log d}{n}}\left( \sqrt{\frac{s\log d}{n} } \right) ^{t}.$$
\end{proposition}

\begin{proposition}\label{spa}
Under the same assumptions and $n\gtrsim d$, we have the estimator $\widehat{\boldsymbol{\theta}}_{t}$ sparse, that is, with high probability,  for some $C > 0$, $\left\|\widehat{\boldsymbol{\theta}}_{t}\right\|_0\leq Cs$.
\end{proposition}	

\begin{proof}[Proof of Proposition~\ref{thm0}]
	When $t=1$, due to the definition of $\widehat{\boldsymbol{\theta}}_{1}$, we have
	\begin{equation}\label{eq100}
		\widehat{\mathcal{L}}\left(\widehat{\boldsymbol{\theta}}_{1}\right)-\widehat{\mathcal{L}}\left(\boldsymbol{\theta}^{\star}\right)
		+\lambda\left\|\widehat{\boldsymbol{\theta}}_{1}\right\|_{1}
		-\lambda\left\|\boldsymbol{\theta}^{\star}\right\|_{1} \leq 0,
	\end{equation}
	and there exists $ \boldsymbol{\xi}\in\left. \partial\left\| \boldsymbol{\theta}\right\|_{1}\right| _{\boldsymbol{\theta}=\boldsymbol{\theta}^{\star}} $, such that
	\begin{equation}\label{kkt}
	\nabla \widehat{\mathcal{L}}(\widehat{\boldsymbol{\theta}_{1}})+\lambda\boldsymbol{\xi}=\boldsymbol{0}.
	\end{equation}
	Since 
	$$
	\begin{aligned}
		\widehat{L}_1(\widehat{\boldsymbol{\theta}}_{1})-\widehat{L}_1(\boldsymbol{\theta}^{\star})-&\left\langle \nabla \widehat{L}_1(\boldsymbol{\theta}^{\star}),\widehat{\boldsymbol{\theta}}_{1}-\boldsymbol{\theta}^{\star}\right\rangle\\=&\frac{1}{n}\left(\widehat{\boldsymbol{\theta}}_{1}-\boldsymbol{\theta}^{\star}\right)^{\top}\sum_{i =1}^{n}\ell_a^{\prime\prime}\left(y_i-\boldsymbol{x}_i^{\top} \widetilde{\boldsymbol{\theta}}\right) \boldsymbol{x}_{i}\boldsymbol{x}_{i}^{\top}\left(\widehat{\boldsymbol{\theta}}_{1}-\boldsymbol{\theta}^{\star}\right)>0,
	\end{aligned}$$
	where $\widetilde{\boldsymbol{\theta}}:=\boldsymbol{\theta}^{\star}+\gamma\left(\widehat{\boldsymbol{\theta}}_{1}-\boldsymbol{\theta}^{\star}\right),\gamma\in (0,1)$, we have that 
	$$
	\begin{aligned}
		\widehat{\mathcal{L}}\left(\widehat{\boldsymbol{\theta}}_{1}\right)-\widehat{\mathcal{L}}\left(\boldsymbol{\theta}^{*}\right)&=\widehat{L}_1(\widehat{\boldsymbol{\theta}}_{1})-\widehat{L}_1(\boldsymbol{\theta}^{\star})-
		\left\langle \nabla \widehat{L}_1(\widehat{\boldsymbol{\theta}}_{0})-\nabla \widetilde{\mathcal{L}}(\widehat{\boldsymbol{\theta}}_{0}),\widehat{\boldsymbol{\theta}}_{1}-\boldsymbol{\theta}^{\star}\right\rangle\\& \geq \left\langle \nabla \widehat{L}_1(\boldsymbol{\theta}^{\star}),\widehat{\boldsymbol{\theta}}_{1}-\boldsymbol{\theta}^{\star}\right\rangle-\left\langle \nabla \widehat{L}_1(\widehat{\boldsymbol{\theta}}_{0})-\nabla \widetilde{\mathcal{L}}(\widehat{\boldsymbol{\theta}}_{0}),\widehat{\boldsymbol{\theta}}_{1}-\boldsymbol{\theta}^{\star}\right\rangle \\& =\left\langle \nabla \widehat{\mathcal{L}}(\boldsymbol{\theta}^{\star}),\widehat{\boldsymbol{\theta}}_{1}-\boldsymbol{\theta}^{\star}\right\rangle \geq -\left\| \nabla \widehat{\mathcal{L}}(\boldsymbol{\theta}^{\star})\right\|_{\infty}\left\|\widehat{\boldsymbol{\theta}}_{1}-\boldsymbol{\theta}^{\star}\right\|_{1}\geq-\frac{\lambda}{2}\left\|\widehat{\boldsymbol{\theta}}_{1}-\boldsymbol{\theta}^{\star}\right\|_{1},
	\end{aligned}
	$$
	where we use the fact that, $ \nabla \widehat{\mathcal{L}}(\boldsymbol{\theta})= \nabla \widehat{L}_1(\boldsymbol{\theta})-\nabla \widehat{L}_1(\widehat{\boldsymbol{\theta}}_{0})+\nabla \widetilde{\mathcal{L}}(\widehat{\boldsymbol{\theta}}_{0})$, for any $ \boldsymbol{\theta} $, and that $ \lambda\geq2\left\| \nabla \widehat{\mathcal{L}}(\boldsymbol{\theta}^{\star})\right\|_{\infty} $.
	Hence, by (\ref{eq100}), we derive that,
	$$-\frac{\lambda}{2}\left\|\widehat{\boldsymbol{\theta}}_{1}-\boldsymbol{\theta}^{\star}\right\|_{1}\leq 
	\lambda\left\|\boldsymbol{\theta}^{*}\right\|_{1}-\lambda\left\|\widehat{\boldsymbol{\theta}}_{1}\right\|_{1}.$$
	Define
	$\widehat{\Delta}:=\widehat{\boldsymbol{\theta}}_{1}-\boldsymbol{\theta}^{\star}$, then $\left\|\widehat{\Delta}_{S^{c}}+\boldsymbol{\theta}^{*}\right\|_{1}=\left\|\widehat{\Delta}_{S^{c}}\right\|_{1}+\left\|\boldsymbol{\theta}^{*}\right\|_{1}$, and we have
	$$
	\begin{aligned}
		-\frac{\lambda}{2}\left\|\widehat{\Delta}_{S}\right\|_{1}-\frac{\lambda}{2}\left\|\widehat{\Delta}_{S^{c}}\right\|_{1}&\leq 
		\lambda\left\|\boldsymbol{\theta}^{*}\right\|_{1}-\lambda\left\|\widehat{\Delta}_{S}+\widehat{\Delta}_{S^{c}}+\boldsymbol{\theta}^{*}\right\|_{1}\\&\leq
		\lambda\left\|\boldsymbol{\theta}^{*}\right\|_{1}-\lambda\left\|\widehat{\Delta}_{S^{c}}+\boldsymbol{\theta}^{*}\right\|_{1}+\lambda\left\|\widehat{\Delta}_{S}\right\|_{1}\leq \lambda\left\|\widehat{\Delta}_{S}\right\|_{1}-\lambda\left\|\widehat{\Delta}_{S^{c}}\right\|_{1}.
	\end{aligned}
	$$
	Hence, $\left\|\widehat{\Delta}_{S^{c}}\right\|_{1}\leq 3 \left\|\widehat{\Delta}_{S}\right\|_{1}$, $\widehat{\boldsymbol{\theta}}_{1}\in \mathcal{C}(\Delta,3)$. We have 
	\begin{equation}\label{1s2}
	\left\|\widehat{\Delta}\right\|_{1}=\left\|\widehat{\Delta}_{S}\right\|_{1}+\left\|\widehat{\Delta}_{S^{c}}\right\|_{1}\leq 4\left\|\widehat{\Delta}_{S}\right\|_{1}\leq 4\sqrt{s}\left\|\widehat{\Delta}\right\|_{2} .
\end{equation} 
Define $\overline{\boldsymbol{\theta}}:=\boldsymbol{\theta}^{\star}+r\widehat{\boldsymbol{\Delta}}
	$, then $\overline{\Delta}:=\overline{\boldsymbol{\theta}}-\boldsymbol{\theta}^{\star}$, where $$r=\begin{cases}
		1, &\text{if $\|\widehat{\boldsymbol{\Delta}}\|_{2} \leq 1$,}\\
		1/\|\widehat{\boldsymbol{\Delta}}\|_{2}, &\text{if $\|\widehat{\boldsymbol{\Delta}}\|_{2} > 1$.}
	\end{cases}
	$$
	Then, by Lemma F.2 of \cite{fan18} and Lemma 2 of \cite{supp}, with probability at least $ 1-C_{1}\exp(-C_{2}n) $, we have
	$$
	\begin{aligned}
	r\langle\nabla \widehat{\mathcal{L}}(\widehat{\boldsymbol{\theta}}_{1})-\nabla \widehat{\mathcal{L}}(\boldsymbol{\theta}^{\star}),\widehat{\Delta}\rangle\geq \langle\nabla \widehat{\mathcal{L}}(\overline{\boldsymbol{\theta}})-\nabla \widehat{\mathcal{L}}(\boldsymbol{\theta}^{\star}),\overline{\Delta}\rangle&\ge c_{1}\|\overline{\Delta}\|_2^2-c_{2}\frac{\log d}{n}\|\overline{\Delta}\|_{1}^2 \\& \geq c_{1}\|\overline{\Delta}\|_2^2-c_{2}s\frac{\log d}{n}\|\overline{\Delta}\|_{2}^2\geq c\|\overline{\Delta}\|_2^2.
	\end{aligned}
	$$
	By (\ref{kkt}) and $r\| \widehat{\Delta}\| _{2}=\| \overline{\Delta}\| _{2}$, we have
	\begin{align*}
	6\sqrt{s}\lambda\|\overline{\Delta}\|_{2}=6\sqrt{s}\lambda r\|\widehat{\Delta}\|_{2}&\geq r\left( \|\lambda\boldsymbol{\xi}\|_{\infty}+\dfrac{\lambda}{2}\right)\|\widehat{\Delta}\|_{1} \\&\geq r \|\lambda\boldsymbol{\xi}+\nabla \widehat{\mathcal{L}}(\boldsymbol{\theta}^{\star})\|_{\infty}\|\widehat{\Delta}\|_{1}\\&\geq r\langle\nabla \widehat{\mathcal{L}}(\widehat{\boldsymbol{\theta}}_{1})-\nabla \widehat{\mathcal{L}}(\boldsymbol{\theta}^{\star}),\widehat{\Delta}\rangle\geq \kappa_{r}\|\overline{\Delta}\|_2^2.
	\end{align*}
	We obtain $\|\overline{\Delta}\|_2\lesssim \sqrt{s}\lambda$. By Proposition~\ref{star} and Proposition~\ref{star'}, since $n\gtrsim \log d$, we have $\|\overline{\Delta}\|_2<1$, meaning that
	 $r=1 $, that is,
	$$P\left(\left\|\widehat{\boldsymbol{\theta}}_{1}-\boldsymbol{\theta}^{\star}\right\|_{2}\leq C\sqrt{s}\lambda\right)\geq 1-e^{-cn}.$$
	Hence we proved the second result, and the first result is obtained by (\ref{1s2}).
\end{proof}

\begin{proof}[Proof of Proposition~\ref{star}]
	$$
	\begin{aligned}
	&\left\| \nabla \widehat{\mathcal{L}}\left(\boldsymbol{\theta}^{\star},\widehat{\boldsymbol{\theta}}_{t}\right)\right\|_{\infty}=\left\|\nabla \widehat{L}_1(\boldsymbol{\theta}^{\star})-\nabla \widehat{L}_1(\widehat{\boldsymbol{\theta}}_{t})+\nabla \widetilde{\mathcal{L}}(\widehat{\boldsymbol{\theta}}_{t})\right\|_{\infty}\\
	=&
	\left\|\frac{1}{n} \sum_{i=1}^n\left(\ell_a^{\prime}(y_i-\boldsymbol{x}_i^{\top} \boldsymbol{\theta}^{\star})-\ell_a^{\prime}(y_i-\boldsymbol{x}_i^{\top} \widehat{\boldsymbol{\theta}}_{t})\right)  \boldsymbol{x}_i+\mathop{\text{Aggregate}}\{\widehat{\boldsymbol{g}}^{(k)}\left(\widehat{\boldsymbol{\theta}}_{t}\right)\}\right\|_{\infty}\\ 
	\leq&
	\left\|\frac{1}{n} \sum_{i=1}^n\left(\ell_a^{\prime}(y_i-\boldsymbol{x}_i^{\top} \boldsymbol{\theta}^{\star})-\ell_a^{\prime}(y_i-\boldsymbol{x}_i^{\top} \widehat{\boldsymbol{\theta}}_{t})\right)  \boldsymbol{x}_i-\frac{1}{N} \sum_{i=1}^N\left(\ell_a^{\prime}(y_i-\boldsymbol{x}_i^{\top} \boldsymbol{\theta}^{\star})-\ell_a^{\prime}(y_i-\boldsymbol{x}_i^{\top} \widehat{\boldsymbol{\theta}}_{t})\right)  \boldsymbol{x}_i\right\|_{\infty}\\
	+&\left\|\frac{1}{N}\sum_{i=1}^N\left(\ell_a^{\prime}(y_i-\boldsymbol{x}_i^{\top} \boldsymbol{\theta}^{\star})-\ell_a^{\prime}(y_i-\boldsymbol{x}_i^{\top} \widehat{\boldsymbol{\theta}}_{t})\right)  \boldsymbol{x}_i-\mathbb{E}\left[\left(\ell_a^{\prime}(y-\boldsymbol{x}^{\top} \boldsymbol{\theta}^{\star})-\ell_a^{\prime}(y-\boldsymbol{x}^{\top} \widehat{\boldsymbol{\theta}}_{t})\right)  \boldsymbol{x}\right]\right\|_{\infty} \\+&\left\|\mathop{\text{Aggregate}}\limits_{k \in [m]}\left\lbrace\widehat {\boldsymbol{g}}^{(k)} \left(\widehat{\boldsymbol{\theta}}_{t}\right)\right\rbrace -\boldsymbol{g}\left(\widehat{\boldsymbol{\theta}_t}\right)\right\|_{\infty}:=T_{1}+T_{2}+T_{3}.
	\end{aligned}
	$$
	For $T_{1}$, there exists a number $\widetilde{u}_{i}$ between $y_i-\boldsymbol{x}_i^{\top} \widehat{\boldsymbol{\theta}}_{t}$ and $y_i-\boldsymbol{x}_i^{\top} \boldsymbol{\theta}^{\star}$, such that,
	$$
	\ell_a^{\prime}(y_i-\boldsymbol{x}_i^{\top} \widehat{\boldsymbol{\theta}}_{t})-\ell_a^{\prime}(y_i-\boldsymbol{x}_i^{\top} \boldsymbol{\theta}^{\star})=\ell_a^{\prime\prime}(y_i-\boldsymbol{x}_i^{\top} \boldsymbol{\theta}^{\star})\boldsymbol{x}_i^{\top} \left(\boldsymbol{\theta}^{\star}-\widehat{\boldsymbol{\theta}}_{t}\right)+\ell_a^{\prime\prime\prime}(\widetilde{u}_{i})\left(\boldsymbol{x}_i^{\top} \left(\boldsymbol{\theta}^{\star}-\widehat{\boldsymbol{\theta}}_{t}\right)\right)^2/2.
	$$
	Thus, we have
	$$
	\begin{aligned}
		T_1&\leq\left\|\frac{1}{n} \sum_{i=1}^n\ell_a^{\prime\prime}(y_i-\boldsymbol{x}_i^{\top} \boldsymbol{\theta}^{\star}) \left\langle\boldsymbol{\theta}^{\star}-\widehat{\boldsymbol{\theta}}_{t}, \boldsymbol{x}_i\right\rangle \boldsymbol{x}_i-\frac{1}{N} \sum_{i=1}^N\ell_a^{\prime\prime}(y_i-\boldsymbol{x}_i^{\top} \boldsymbol{\theta}^{\star}) \left\langle\boldsymbol{\theta}^{\star}-\widehat{\boldsymbol{\theta}}_{t}, \boldsymbol{x}_i\right\rangle \boldsymbol{x}_i\right\|_{\infty}\\&+\left\|\frac{1}{n} \sum_{i=1}^n\frac{\ell_a^{\prime\prime\prime}(\widetilde{u}_{i})}{2} \left\langle\boldsymbol{\theta}^{\star}-\widehat{\boldsymbol{\theta}}_{t}, \boldsymbol{x}_i\right\rangle^2\boldsymbol{x}_i-\frac{1}{N} \sum_{i=1}^N\frac{\ell_a^{\prime\prime\prime}(\widetilde{u}_{i})}{2} \left\langle\boldsymbol{\theta}^{\star}-\widehat{\boldsymbol{\theta}}_{t}, \boldsymbol{x}_i\right\rangle^2\boldsymbol{x}_i\right\|_{\infty}:=T_{11}+T_{12}.
	\end{aligned}
	$$
	With probability at least $ 1- d^{-C}$,
	$$
	\begin{aligned}
	T_{11}&\leq \left\|\frac{1}{n} \sum_{i=1}^n\ell_a^{\prime\prime}(y_i-\boldsymbol{x}_i^{\top} \boldsymbol{\theta}^{\star}) \boldsymbol{x}_i\boldsymbol{x}_i^{\top}-\frac{1}{N} \sum_{i=1}^N\ell_a^{\prime\prime}(y_i-\boldsymbol{x}_i^{\top} \boldsymbol{\theta}^{\star}) \boldsymbol{x}_i \boldsymbol{x}_i^{\top}\right\|_{\infty}\left\|\boldsymbol{\theta}^{\star}-\widehat{\boldsymbol{\theta}}_{t}\right\|_{1}\\ & \leq 
	\left(\left\|\frac{1}{n} \sum_{i=1}^n\ell_a^{\prime\prime}(y_i-\boldsymbol{x}_i^{\top} \boldsymbol{\theta}^{\star}) \boldsymbol{x}_i\boldsymbol{x}_i^{\top}-\mathbb{E}\left[\ell_a^{\prime\prime}(y_i-\boldsymbol{x}_i^{\top} \boldsymbol{\theta}^{\star}) \boldsymbol{x}_i\boldsymbol{x}_i^{\top}\right]\right\|_{\infty}\right. \\&\left. +\left\|\frac{1}{N} \sum_{i=1}^N\ell_a^{\prime\prime}(y_i-\boldsymbol{x}_i^{\top} \boldsymbol{\theta}^{\star}) \boldsymbol{x}_i\boldsymbol{x}_i^{\top}-\mathbb{E}\left[\ell_a^{\prime\prime}(y_i-\boldsymbol{x}_i^{\top} \boldsymbol{\theta}^{\star}) \boldsymbol{x}_i\boldsymbol{x}_i^{\top}\right]\right\|_{\infty}\right) \left\|\boldsymbol{\theta}^{\star}-\widehat{\boldsymbol{\theta}}_{t}\right\|_{1}\\ &\leq 
	C\sqrt{s}d_{t}\left( \frac{\log d}{n}+\sqrt{\frac{\log d}{n}}+\frac{\log d}{N}+\sqrt{\frac{\log d}{N}}\right) \lesssim \sqrt{s}d_{t}\sqrt{\frac{\log d}{n}} ,
    \end{aligned}
	$$
	where we used that $ \left\|\frac{1}{n} \sum_{i=1}^n\ell_a^{\prime\prime}(y_i-\boldsymbol{x}_i^{\top} \boldsymbol{\theta}^{\star}) \boldsymbol{x}_i\boldsymbol{x}_i^{\top}-\mathbb{E}\left[\ell_a^{\prime\prime}(y_i-\boldsymbol{x}_i^{\top} \boldsymbol{\theta}^{\star}) \boldsymbol{x}_i\boldsymbol{x}_i^{\top}\right]\right\|_{\infty}\leq C\left( \frac{\log d}{n}+\sqrt{\frac{\log d}{n}}\right)  $ by Bernstein's inequality and our assumptions, in the third inequality. And by the properties of function $ \ell_a $, we have
	$$
	\begin{aligned}
	T_{12}&\leq \left\|\frac{1}{n} \sum_{i=1}^n\frac{\ell_a^{\prime\prime\prime}(\widetilde{u}_{i})}{2} \left\langle\boldsymbol{\theta}^{\star}-\widehat{\boldsymbol{\theta}}_{t}, \boldsymbol{x}_i\right\rangle^2\boldsymbol{x}_i\right\|_{\infty}+\left\|\frac{1}{N} \sum_{i=1}^N\frac{\ell_a^{\prime\prime\prime}(\widetilde{u}_{i})}{2} \left\langle\boldsymbol{\theta}^{\star}-\widehat{\boldsymbol{\theta}}_{t}, \boldsymbol{x}_i\right\rangle^2\boldsymbol{x}_i\right\|_{\infty}\\ &\lesssim  a\left\|\boldsymbol{\theta}^{\star}-\widehat{\boldsymbol{\theta}}_{t}\right\|_{1}^2 \|\boldsymbol{x}_i\|_{\infty}^3\lesssim sd_{t}^{2}\frac{\log^{3/2} d}{\sqrt{n} }.
	\end{aligned}
	$$
	For $T_2$, by Lemma \ref{unibd}, we immediately have that $T_2\leq C\sqrt{s}d_t \sqrt{\frac{\log^3 d}{N}}$ with probability at least $ 1-d^{-C} $.

	For Option \uppercase\expandafter{\romannumeral1},
	by Lemma~\ref{trmean}, $ T_3\leq  \frac{1}{(1-2 \beta) m}\left\|\sum_{k \in \mathcal{M}}\left(\widehat{\boldsymbol{g}}^{(k)}-\boldsymbol{g}\right)\right\|_\infty+\frac{2\beta+\alpha}{1-2 \beta}\max_{k \in  \mathcal{M}}\left\|\widehat{\boldsymbol{g}}^{(k)}-\boldsymbol{g}\right\|_{\infty} . $	
		
	$$
	\begin{aligned}
	&\left\|\frac{1}{(1-\alpha) m} \sum_{k \in\mathcal{M}}\left[\nabla\widehat{L}_k\left(\widehat{\boldsymbol{\theta}}_{t}\right)\right]-\mathbb{E}\left[\frac{1}{(1-\alpha) m} \sum_{k \in\mathcal{M}}\nabla\widehat{L}_k\left(\widehat{\boldsymbol{\theta}}_{t}\right)\right]\right\|_{\infty} \\ \leq & \left\|\frac{1}{(1-\alpha) m} \sum_{k \in\mathcal{M}}\left[\nabla\widehat{L}_k\left(\widehat{\boldsymbol{\theta}}_{t}\right)-\nabla \widehat{L}_k(\boldsymbol{\theta}^{\star})\right]-\frac{1}{(1-\alpha) m} \sum_{k \in\mathcal{M}}\mathbb{E}\left[\nabla\widehat{L}_k\left(\widehat{\boldsymbol{\theta}}_{t}\right)-\nabla \widehat{L}_k(\boldsymbol{\theta}^{\star})\right]\right\|_{\infty}\\+&
	\left\|\frac{1}{(1-\alpha) m} \sum_{k \in\mathcal{M}}\left[\nabla \widehat{L}_k(\boldsymbol{\theta}^{\star})\right]-\mathbb{E}\left[\nabla \widehat{L}_k(\boldsymbol{\theta}^{\star})\right]\right\|_{\infty}=:T_{31}+T_{32}.
	\end{aligned}
	$$
	Using Lemma~\ref{unibd}, we get $ T_{31} \lesssim  \sqrt{s}d_{t}\sqrt{\log d} \sqrt{\frac{s\log d }{(1-\alpha)mn}} $, with probability at least $ 1-d^{-C} $. And 
	$T_{32}=\max_{j\in[d]}\left|\frac{1}{(1-\alpha)N} \sum\limits_{i\in [n],k \in \mathcal{M}}x_{i,j}^{(k)}\ell_{a}^{\prime}(y_{i}-\boldsymbol{x}_{i}^{\top}\boldsymbol{\theta}^{\star})-\mathbb{E}\left[x_{\cdot,j}\ell_{a}^{\prime}(y-\boldsymbol{x}^{\top}\boldsymbol{\theta}^{\star})\right]\right|$.
Since $\left|\ell_{a}^{\prime}(y-\boldsymbol{x}^{\top}\boldsymbol{\theta}^{\star})\right|\leq 2/a, $ $\left\|x_{i,j}^{(k)}\ell_{a}^{\prime}(y_{i}-\boldsymbol{x}_{i}^{\top}\boldsymbol{\theta}^{\star})\right\|_{\psi_{2}}\leq \frac{2}{a}\|x_{i,j}^{(k)}\|_{\psi_{2}}=\frac{2}{a}K$, by Bernstein's inequality, we have $$
T_{32}\leq \frac{CK}{a}\left(\sqrt{\frac{\log d}{(1-\alpha)N}}+\frac{\log d}{(1-\alpha)N}\right)  \lesssim \sqrt{\frac{\log d}{(1-\alpha)N}}
,$$ with probability at least $1-2d^{-C}$. 

	Similarly, with probability at least $ 1-md^{-C} $,
	$$\max_{k \in\mathcal{M}}\left|\left[\nabla\widehat{L}_k\left(\widehat{\boldsymbol{\theta}}_{t}\right)\right]_{j}-\mathbb{E}\left[\nabla\widehat{L}_k\left(\widehat{\boldsymbol{\theta}}_{t}\right)\right]_{j}\right|\leq (1+s d_{t}\sqrt{\log d})\sqrt{\frac{\log d}{n}}$$
	Therefore, we derive that, with probability at least $ 1-md^{-C}, $
	$$T_{3}\lesssim(1+s d_{t}\sqrt{\log d})\left( 
	\frac{\sqrt{1-\alpha}}{1-2\beta}\sqrt{\frac{\log  d}{N}}+\frac{2\beta+\alpha}{1-2\beta}\sqrt{\frac{\log d}{n}}\right) .$$

	Hence, with probability at least $  1-md^{-C}, $
	$$\left\| \nabla \widehat{\mathcal{L}}\left(\boldsymbol{\theta}^{\star},\widehat{\boldsymbol{\theta}}_{t}\right)\right\|_{\infty}\leq C \left(\sqrt{\frac{\log d}{N}}+\beta \sqrt{\frac{\log d}{n}}+d_t\cdot\left(\frac{s\log d}{\sqrt{N} }+\frac{s\log d}{\sqrt{n}}\right) \right) .$$
	
	By math induction, we finally obtain with probability at least $  1-tmd^{-C}, $
	$$\left\| \nabla \widehat{\mathcal{L}}\left(\boldsymbol{\theta}^{\star},\widehat{\boldsymbol{\theta}}_{t}\right)\right\|_{\infty}\lesssim \sqrt{\frac{\log  d}{N}}+\beta \sqrt{\frac{\log d}{n}}+\sqrt{\frac{\log d}{n}}\left( \frac{s\log d}{\sqrt{n}}\right) ^{t}.$$
\end{proof}

\begin{proof}[Proof of Proposition \ref{star'}]
	
	For  Option \uppercase\expandafter{\romannumeral2}, the preceding decomposition and proof align with those of Proposition~\ref{star}, with the only distinction lying in the handling of the term \( T_3 \). By the results of Lemma 1 in \cite{stat}, which explicitly bound the tails of the empirical distribution function for honest samples, we obtain that, with probability at least $1-4e^{-2t}$,
	\begin{equation}\label{med}
		\left| Q_{\frac{1}{2}+\alpha}\left\{\widehat{g}^{(k)}_j\left(\boldsymbol{\theta}\right):k  \in \mathcal{M}\right\} - g_j\left(\boldsymbol{\theta}\right) \right| \leq C_\epsilon \sqrt{\frac{\text{Var}(\tilde{\ell}_j(y-\boldsymbol{x}^{\top}\boldsymbol{\theta}))}{n}} \left( \alpha + \sqrt{\frac{t}{m(1-\alpha)}} + c \frac{\gamma(\tilde{\ell}_j(y-\boldsymbol{x}^{\top}\boldsymbol{\theta}))}{\sqrt{n}} \right)
	\end{equation}
	for any fixed $\boldsymbol{\theta} \in \Omega=\left\{\boldsymbol{\theta}:\|\boldsymbol{\theta}\|_0 \leq C s,\left\|\boldsymbol{\theta}-\boldsymbol{\theta}^{\star}\right\|_1 \leq R\right\}$ and $j\in[d]$, where $C_\epsilon$ is defined as $C_\epsilon := \sqrt{2\pi} \exp\left( \frac{1}{2} \left( \Phi^{-1}(1-\epsilon) \right)^2 \right),$ with $\Phi^{-1}(\cdot)$ being the inverse of the cumulative distribution function $\Phi(\cdot)$ of the standard Gaussian distribution. We have $\text{Var}(\tilde{\ell}_j(y-\boldsymbol{x}^{\top}\boldsymbol{\theta}))\leq (2/a)^2$, and  
$\gamma(\tilde{\ell}_j(y-\boldsymbol{x}^{\top}\boldsymbol{\theta})) := \mathbb{E}[ | \tilde{\ell}_j(y-\boldsymbol{x}^{\top}\boldsymbol{\theta}) - \mathbb{E}[\tilde{\ell}_j(y-\boldsymbol{x}^{\top}\boldsymbol{\theta})] |^3 ]/\text{Var}(\tilde{\ell}_j(y-\boldsymbol{x}^{\top}\boldsymbol{\theta}))^{3/2}\leq S< \infty$, since $\tilde{\ell}_j(y-\boldsymbol{x}^{\top}\boldsymbol{\theta})$ is sub-gaussian.  And the same is true for $Q_{1/2-\alpha} $. 

	Let $U^{(k)}(\boldsymbol{\theta})=\widehat{\boldsymbol{g}}^{(k)}(\boldsymbol{\theta})- \widehat{\boldsymbol{g}}^{(k)}( \boldsymbol{\theta}^{\star})-\left(\boldsymbol{g}(\boldsymbol{\theta})-\boldsymbol{g}( \boldsymbol{\theta}^{\star})\right)$. Then, we have the following identity:
	\begin{equation}\label{eq:U}
		\widehat{\boldsymbol{g}}^{(k)}( \widehat{\boldsymbol{\theta}}_t) = 
	\widehat{\boldsymbol{g}}^{(k)}( \boldsymbol{\theta}^{\star})-\boldsymbol{g}( \boldsymbol{\theta}^{\star})+\boldsymbol{g}(\widehat{\boldsymbol{\theta}}_t)+U^{(k)}( \widehat{\boldsymbol{\theta}}_t).
	\end{equation}
	Hence, we derive that, 
$$
\begin{aligned}
T_3 &= \left\|\operatorname*{median}_{k \in [m]}\left\{\widehat{\boldsymbol{g}}^{(k)}\left(\boldsymbol{\theta}^\star\right)-\boldsymbol{g}( \boldsymbol{\theta}^{\star})+U^{(k)}(\widehat{\boldsymbol{\theta}}_t)\right\}\right\|_\infty \\
&\leq \max\left\{\left\|Q_{\frac{1}{2}-\alpha}\left\{\widehat{\boldsymbol{g}}^{(k)}\left(\boldsymbol{\theta}^\star\right):k  \in \mathcal{M}\right\}-\boldsymbol{g}( \boldsymbol{\theta}^{\star})\right\|_\infty,\left\|Q_{\frac{1}{2}+\alpha}\left\{\widehat{\boldsymbol{g}}^{(k)}\left(\boldsymbol{\theta}^\star\right):k  \in \mathcal{M}\right\}-\boldsymbol{g}( \boldsymbol{\theta}^{\star})\right\|_\infty\right\}\\
&+\max_{k\in \mathcal{M}}\sup_{\boldsymbol{\theta}\in \Omega}\left\|U^{(k)}\left(\boldsymbol{\theta}\right)\right\|_\infty\\
&:=T'_{31}+T'_{32},
\end{aligned}
$$
where the inequality follows from Lemma 3 and Equation (38) in \cite{Divide}. 
For $T'_{32}$, by Lemma \ref{unibd}, we obtain that $T'_{32}\leq CR\sqrt{\frac{\log^3 d}{n}}$ with probability at least $ 1-(1-\alpha)md^{-C} $.
And for $T'_{31}$, using equation (\ref{med}) and union bound for $j\in[d]$, we have that, with probability at least $1-4d \exp{(-2t)}$, 
$$
T'_{31}\leq C\frac{2}{a}\left(\sqrt{\frac{t}{(1-\alpha)N}}+\frac{\alpha}{\sqrt{n}}+\frac{S}{n}\right),
$$
substitute $t=C\log d$ and $R\leq\sqrt{s}d_t$ into the above inequalities, we obtain that, with probability at least $1-md^{-C}$, 
$$
T_3\lesssim \sqrt{\frac{\log d}{N}}+\frac{\alpha}{\sqrt{n}}+\frac{1}{n}+d_t\sqrt{\frac{s\log d}{n}}.
$$

	Hence, with probability at least $  1-md^{-C}, $
	$$\left\| \nabla \widehat{\mathcal{L}}\left(\boldsymbol{\theta}^{\star},\widehat{\boldsymbol{\theta}}_{t}\right)\right\|_{\infty}\leq C \left(\sqrt{\frac{\log  d}{N}}+\frac{\alpha}{\sqrt{n} }+\frac{1}{n}+d_t\cdot\sqrt{\frac{s\log d}{n} } \right) .$$
	
	By math induction, we finally obtain with probability at least $  1-tmd^{-C}, $
	$$\left\| \nabla \widehat{\mathcal{L}}\left(\boldsymbol{\theta}^{\star},\widehat{\boldsymbol{\theta}}_{t}\right)\right\|_{\infty}\lesssim \sqrt{\frac{\log  d}{N}}+\frac{\alpha}{\sqrt{n} }+\frac{1}{n}+\sqrt{\frac{\log d}{n}}\left( \sqrt{\frac{s\log d}{n} } \right) ^{t}.$$
\end{proof}

\subsubsection{Proof of Proposition~\ref{spa}}
\begin{proof}
	By definition, there exists $\boldsymbol{\xi}_1\in\partial\left\| \widehat{\boldsymbol{\theta}}_1\right\|_{1}$ such that

$$\lambda \boldsymbol{\xi}_{1} = \frac{1}{n} \sum_{i=1}^{n} \ell_{a}^{\prime}(y_{i} - \langle \boldsymbol{x}_{i}, \widehat{\boldsymbol{\theta}}_1 \rangle) \boldsymbol{x}_{i} - \frac{1}{n} \sum_{i=1}^{n} \ell_{a}^{\prime}(y_{i} - \langle \boldsymbol{x}_{i}, \widehat{\boldsymbol{\theta}}_{0} \rangle) \boldsymbol{x}_{i} -\mathop{\text{Aggregate}}\limits_{k \in [m]}\left\lbrace\widehat {\boldsymbol{g}}^{(k)} \left(\widehat{\boldsymbol{\theta}}_{0}\right)\right\rbrace.$$
Let $\hat{S}=\{j\in[d]:\widehat{\boldsymbol{\theta}}_{1,j}\neq 0\}$ and $\hat{s}=|\hat{S}|$. We know that $$\begin{cases}\boldsymbol{\xi}_{1,j} = \text{sign}(\boldsymbol{\theta}_{1,j})&\text{if } j\in\hat{S},\\
\boldsymbol{\xi}_{1,j} \in [-1, 1]&\text{if }j\notin\hat{S}.\end{cases}$$
Hence, $$\nabla \widehat{\mathcal{L}}(\widehat{\boldsymbol{\theta}}_0)_{\hat{S}} + \lambda \cdot \text{sign}\left(\widehat{\boldsymbol{\theta}}_{1,\hat{S}}\right) = 0,$$
and we have, \begin{equation}\label{support}
	\lambda\sqrt{\hat{s}}\leq \| \nabla \widehat{\mathcal{L}}(\widehat{\boldsymbol{\theta}}_0)_{\hat{S}} \|_2 .
\end{equation}
Since $\nabla \widehat{\mathcal{L}}(\widehat{\boldsymbol{\theta}}_0) - \nabla \widehat{\mathcal{L}}(\boldsymbol{\theta}^*) = \nabla L_1(\widehat{\boldsymbol{\theta}}_0) - \nabla L_1(\boldsymbol{\theta}^*)= \mathbf{H}_{mean}\widehat{\Delta}$,  where the second equality we use the second mean value theorem, $\widehat{\Delta}=\widehat{\boldsymbol{\theta}}_{0}-\boldsymbol{\theta}^{\star}$ and $\mathbf{H}_{mean}$ is the Hessian matrix for some $\boldsymbol{\theta}$.  we derive that,$$ \| (\nabla \widehat{\mathcal{L}}(\widehat{\boldsymbol{\theta}}_0)_{\hat{S}}  \|_2 \le \| (\nabla \widehat{\mathcal{L}}(\boldsymbol{\theta}^\star))_{\hat{S}} \|_2+ \| (\mathbf{H}_{mean} \widehat{\Delta})_{\hat{S}} \|_2,$$ the Hessian matrix form of the pseudo-Huber loss is: $$ \nabla^2 L_1(\boldsymbol{\theta}) = \frac{1}{n} \mathbf{X}^\top \mathbf{W}(\boldsymbol{\theta}) \mathbf{X}, $$where $\mathbf{W}$ is a diagonal matrix, and the weight $w_i = \ell_a''(y_i-\boldsymbol{x}_i^\top \boldsymbol{\theta} )\leq 2$, then  $\mathbf{H}_{mean} = \frac{1}{n} \mathbf{X}^\top \mathbf{W}(\boldsymbol{u}) \mathbf{X}$ for some $\boldsymbol{u}$ between $\boldsymbol{\theta}^\star$ and $\widehat{\boldsymbol{\theta}}_0$.  By theorem 10.5.11 of \cite{hdp}, we have with probability at least $1-2\exp(-cn)$, when $n\ge Cs\log(ed/s),$
$$\frac{\|\mathbf{X}^\top v\|_2}{\sqrt{n}}\leq c<2,$$
for all sparse vectors $v\in\mathbb{R}^d$, such that, $\|v\|_2=1$, $\|v\|_0\leq s$. Then, we get $$\| (\mathbf{H}_{mean} \widehat{\Delta})_{\hat{S}} \|_2\leq c\|\widehat{\Delta}\|_2.$$
Finally, combined with (\ref{support}), we obtain \begin{align*}
\lambda\sqrt{\hat{s}}&\leq\| (\nabla \widehat{\mathcal{L}}(\boldsymbol{\theta}^\star))_{\hat{S}} \|_2+ \| (\mathbf{H}_{mean} \widehat{\Delta})_{\hat{S}} \|_2\\
&\leq \sqrt{\hat{s}}\| \nabla \widehat{\mathcal{L}}(\boldsymbol{\theta}^\star)\|_\infty + c \|\widehat{\Delta}\|_2\\
&\leq \sqrt{\hat{s}}\frac{\lambda}{2}+c\lambda\sqrt{s},
\end{align*}
which conduces that, $\hat{s}\leq C s$.
\end{proof}

\subsection{Proof for Quantile Regression and SVM}
We mainly provide the part we sovle differently from the proof in \cite{wanglei1}.
\begin{proof}[Proof of Proposition~\ref{theorem3}]
	Due to Theorem 1 of \cite{wanglei1}, for $ \lambda_{t+1}\geq 2\left\| \nabla \widehat{\mathcal{L}}\left(\boldsymbol{\theta}^{\star},\widehat{\boldsymbol{\theta}}_{t}\right)\right\|_{\infty} $, we have with probability at least $ 1-n^{-C} $,
	\begin{equation}\label{thm2}
	d_{t}\leq C\left( \lambda_{t}\sqrt{s}+\frac{s^{3/2}\log d\log n}{n}\right) ,
	\end{equation}
then the convergence rate would be dominated by $ \lambda_{t} $. 
We just need to estimate $\lambda$, whose bound is related to $\|\nabla\widehat{\mathcal{L}}(\boldsymbol{\theta}^\star)\|_\infty$. Since Proposition 1 in \cite{wanglei1} shows that, with probability at least $1-d^{-C}$,
$$ \left\| \nabla \widehat{L}_1(\boldsymbol{\theta}^\star) - \nabla \widehat{L}_1(\widehat{\boldsymbol{\theta}}_0) + \nabla \widehat{L}(\widehat{\boldsymbol{\theta}}_0) \right\|_\infty \leq C \left( \sqrt{\frac{\log d}{N}} + \frac{s \log d}{n^{3/4}} + \frac{s (\log d)^{3/2}}{n} \right), $$
and noticed that,
$$ \|\nabla \widehat{\mathcal{L}}(\boldsymbol{\theta}^\star)\|_\infty \leq \left\| \nabla \widehat{L}_1(\boldsymbol{\theta}^\star) - \nabla \widehat{L}_1(\widehat{\boldsymbol{\theta}}_0) + \nabla \widehat{L}(\widehat{\boldsymbol{\theta}}_0) \right\|_\infty + \left\| \nabla \widehat{L}(\widehat{\boldsymbol{\theta}}_0) -\mathop{\text{Aggregate}}\limits_{k \in [m]}\left\lbrace\widehat {\boldsymbol{g}}^{(k)} \left(\widehat{\boldsymbol{\theta}}_{0}\right)\right\rbrace\right\|_\infty,$$we only need to give a bound of  $\left\| \nabla \widehat{L}(\widehat{\boldsymbol{\theta}}_0) -\mathop{\text{Aggregate}}\limits_{k \in [m]}\left\lbrace\widehat {\boldsymbol{g}}^{(k)} \left(\widehat{\boldsymbol{\theta}}_{0}\right)\right\rbrace\right\|_\infty.$

For Option \uppercase\expandafter{\romannumeral1}, let $\bar{\boldsymbol{g}}=\frac{1}{m}  \sum_{k=1}^{m}\nabla \widehat{L}_k(\widehat{\boldsymbol{\theta}}_0) $,
 $$\left\| \nabla \widehat{L}(\widehat{\boldsymbol{\theta}}_0) -\mathop{\text{trmean}}\limits_{\beta}\left\lbrace\widehat {\boldsymbol{g}}^{(k)} \left(\widehat{\boldsymbol{\theta}}_{0}\right)\right\rbrace\right\|_\infty\leq \frac{1}{(1-2 \beta) m}\left\|\sum_{k \in \mathcal{M}}\left(\widehat{\boldsymbol{g}}^{(k)}-\bar{\boldsymbol{g}}\right)\right\|_ \infty+\frac{2\beta+\alpha}{1-2 \beta}\max_{k \in \mathcal{M}}\left\|\widehat{\boldsymbol{g}}^{(k)}-\bar{\boldsymbol{g}}\right\|_ {\infty} .$$ Since
\begin{align*}
    &\quad\frac{1}{(1-\alpha)m }\sum_{k \in \mathcal{M}}\left(\widehat{\boldsymbol{g}}^{(k)}-\bar{\boldsymbol{g}}\right)\\
    &= \frac{1}{(1-\alpha)m }\sum_{k \in \mathcal{M}}\left(\widehat{\boldsymbol{g}}^{(k)}-\boldsymbol{g}\right)-\frac{1}{(1-\alpha)m }\sum_{k \in \mathcal{M}}\left(\bar{\boldsymbol{g}}-\boldsymbol{g}\right) \\
    &= \frac{1}{(1-\alpha)N}\sum_{1\leq i\leq n, k\in\mathcal{M}} \bigg( \boldsymbol{x}_i^{(k)} \big( \tau - \mathbf{1}_{\left\{ y_i - \widehat{\boldsymbol{\theta}}_0^\top \boldsymbol{x}_i^{(k)} \leq 0 \right\}} \big) - \mathbb{E} \left[ \boldsymbol{x} \big( \tau - \mathbf{1}_{\left\{ y - \widehat{\boldsymbol{\theta}}_0^\top \boldsymbol{x} \leq 0 \right\}} \big) \right] \bigg) \\
    &- \frac{1}{N}\sum_{1\leq i\leq N} \bigg( \boldsymbol{x}_i^{(k)} \big( \tau - \mathbf{1}_{\left\{ y_i - \widehat{\boldsymbol{\theta}}_0^\top \boldsymbol{x}_i^{(k)} \leq 0 \right\}} \big) - \mathbb{E} \left[ \boldsymbol{x} \big( \tau - \mathbf{1}_{\left\{ y - \widehat{\boldsymbol{\theta}}_0^\top \boldsymbol{x} \leq 0 \right\}} \big) \right] \bigg),
\end{align*}
and 
\begin{align*}
\widehat{\boldsymbol{g}}^{(K)}-\bar{\boldsymbol{g}}
&=\frac{m-1}{m}\left(\widehat{\boldsymbol{g}}^{(K)}-\widehat{\boldsymbol{g}}^{(-K)}\right)\\
&=\left(1-\frac{1}{m}\right)\frac{1}{n}\sum_{1\leq i\leq n}\bigg( \boldsymbol{x}_i^{(K)} \left( \tau - \mathbf{1}_{\left\{ y_i - \widehat{\boldsymbol{\theta}}_0^\top \boldsymbol{x}_i^{(K)}\leq 0 \right\}} \right)-\mathbb{E} \left[ \boldsymbol{x} \left( \tau - \mathbf{1}_{\left\{ y - \widehat{\boldsymbol{\theta}}_0^\top \boldsymbol{x}\leq 0 \right\}} \right) \right] \bigg)\\
&- \frac{1}{N}\sum_{k\neq K}\bigg( \boldsymbol{x}_i^{(k)} \left( \tau - \mathbf{1}_{\left\{ y_i - \widehat{\boldsymbol{\theta}}_0^\top \boldsymbol{x}_i^{(k)}\leq 0 \right\}} \right)-\mathbb{E} \left[ \boldsymbol{x} \left( \tau - \mathbf{1}_{\left\{ y - \widehat{\boldsymbol{\theta}}_0^\top \boldsymbol{x}\leq 0 \right\}} \right) \right] \bigg),
\end{align*}
by similar argument as Lemma A.1 of \cite{wanglei1}, with  $$\sigma^2=\sup_j\mathbb{E}\left[ x_{j}\mathbf{1}_{\left\{ \max_j |x_j| \leq c_n \right\}}\big( \tau - \mathbf{1}_{\left\{ y - \widehat{\boldsymbol{\theta}}_0^\top \boldsymbol{x} \leq 0 \right\}} \big) \right]^2\leq \log (Nd),$$ we get
with probability at least $1-d^{-C}$, $$\frac{1}{(1-2 \beta) m}\left\|\sum_{k \in \mathcal{M}}\left(\widehat{\boldsymbol{g}}^{(k)}-\bar{\boldsymbol{g}}\right)\right\|_ \infty\lesssim \frac{1-\alpha}{1-2 \beta} \left( \sqrt{\frac{\log d}{(1-\alpha)N}} +\sqrt{\frac{\log d}{N}} \right),$$ and with probability at least $1-(1-\alpha)md^{-C}$,$$\frac{2\beta+\alpha}{1-2 \beta}\max_{k \in \mathcal{M}}\left\|\widehat{\boldsymbol{g}}^{(k)}-\bar{\boldsymbol{g}}\right\|_ {\infty} \lesssim \frac{2\beta+\alpha}{1-2 \beta}(1-\frac{1}{m})\sqrt{\frac{\log d}{n}}.$$
Put all the above togather, and we proved that, with probability at least $1-(1-\alpha)md^{-C}$, $$\|\nabla \widehat{\mathcal{L}}(\boldsymbol{\theta}^\star)\|_\infty\leq C\left( \sqrt{\frac{\log d}{N}} + \beta \sqrt{\frac{\log d}{n}}  + \frac{s \log d}{n^{3/4}}  \right).$$ 

with probability at least $1-(1-\alpha)tmd^{-C}.$
On that event, with probability at least $1 - n^{-C}$,$$\|\widehat{\boldsymbol{\theta}} - \boldsymbol{\theta}^\star\|_ 2 \lesssim \sqrt{\frac{s\log d}{N}} + \beta \sqrt{\frac{s\log d}{n}} +  \frac{s^{3/2}\log d}{n^{3/4}}  .$$

\end{proof}

The proof for Option \uppercase\expandafter{\romannumeral2} is similar to that of Proposition~\ref{star'}, so we omit the proof.
And in general, by induction, we can prove that,
\[\left\|\widehat{\boldsymbol{\theta}}_t - \boldsymbol{\theta}^\star\right\|_ 2\lesssim  \sqrt{\frac{s\log d}{N}} + \alpha \sqrt{\frac{s\log d}{n}}  + 
\left( \frac{s \log d}{n} \right)^{1 - (1/2)^{t+1}},
\]
for proper choice of $\widehat{\boldsymbol{\theta}}_{0}$. As an extension of results in \cite{wanglei1}, we can see that, the iterative term is still dominated by contamination error term, and the essential order of the error term is $O(\sqrt{s\log d/N})+O(\alpha\sqrt{s\log d/n})$.

In the proof of Theorem~\ref{theorem4}, compared to \cite{svm}, we prove the corresponding estimation bound of the communication error using methods in \cite{svmm}. 

\begin{proof}[Proof of Theorem~\ref{theorem4}]
Assume that $\lambda \geq 2\|\nabla \widehat{\mathcal{L}}(\boldsymbol{\theta}^\star)\|_\infty$, we will fisrt use the following Step 1 -- Step 3 to prove our first conclusion: with probability at least $1 - n^{-C}$, $$ \|\widehat{\boldsymbol{\theta}} - \boldsymbol{\theta}^\star\|_2 \leq C \left[ \lambda \sqrt{s} + \frac{s^{3/2}(\log d)^{5/2}}{n}\right]. $$

\noindent\textbf{Step 1.} Let $\boldsymbol{\delta}=\boldsymbol{\theta}-\boldsymbol{\theta}^\star$, and $\widehat{\boldsymbol{\delta}}=\widehat{\boldsymbol{\theta}}-\boldsymbol{\theta}^\star$, by the same argument as the proof of Theorem~\ref{theorem1}, we know that $$ \|\widehat{\boldsymbol{\delta}}_{S^c}\|_1 \leq 3\|\widehat{\boldsymbol{\delta}}_S\|_1. $$

\noindent\textbf{Step 2.}  We claim that, with probability at least $1-n^{-C}$, 
$$\begin{aligned}
&\sup_{\substack{\|\boldsymbol{\delta}\|_2 \leq t,\,\\\|\boldsymbol{\delta}_{S^c}\|_1 \leq 3\|\boldsymbol{\delta}_S\|_1}} \left| \widehat{L}_1(\boldsymbol{\theta}^\star + \boldsymbol{\delta}) - \widehat{L}_1(\boldsymbol{\theta}^\star) - \boldsymbol{\delta}^\top \nabla \widehat{L}_1(\boldsymbol{\theta}^\star) - \mathbb{E}\left[\widehat{L}_1(\boldsymbol{\theta}^\star + \boldsymbol{\delta})\right] + \mathbb{E}\left[\widehat{L}_1(\boldsymbol{\theta}^\star)\right] \right|\\
& \leq C \left( \frac{s^{3/4} (c_n t)^{3/2} \sqrt{\log n}}{\sqrt{n}} + \frac{c_n \sqrt{s} t \log n}{n} \right).
\end{aligned}  $$
Without loss of generality, we condition on $\{y_i = 1, i = 1, \dots, n\}$ in the following, since all $2^n$ cases can be handled in exactly the same way. We also condition on the event $\{\max_i \|\boldsymbol{x}_i\|_\infty \leq c_n := C\sqrt{\log (nd)}\}$. Using the identity 
$$
(a - b)_+ - a_+ = \int_0^{b} \mathbf{1}_{\{0<a \leq u\}} du - b \mathbf{1}_{\{a > 0\}},
$$
we derive that
$$\left(1 - y \boldsymbol{x}^\top (\boldsymbol{\theta}^\star + \boldsymbol{\delta})\right)_+ - \left(1 - y \boldsymbol{x}^\top \boldsymbol{\theta}^\star\right)_+ +\boldsymbol{\delta}^\top y\boldsymbol{x}\mathbf{1}_{\{1 - y \boldsymbol{x}^\top \boldsymbol{\theta}^\star\geq0\}}=\int_0^{\boldsymbol{x}^\top\boldsymbol{\delta}}\mathbf{1}_{\{0<1 - y \boldsymbol{x}^\top \boldsymbol{\theta}^\star\leq u\}}du,
$$
hence, 
$$\begin{aligned}
\widehat{L}_1(\boldsymbol{\theta}^\star+\boldsymbol{\delta}) - \widehat{L}_1(\boldsymbol{\theta}^\star)-\boldsymbol{\delta}^\top\nabla\widehat{L}_1(\boldsymbol{\theta}^\star)=\frac{1}{n}\sum_{i=1}^{n}\int_0^{\boldsymbol{x}_i^\top\boldsymbol{\delta}}\mathbf{1}_{\{0<1 -  \boldsymbol{x}_i^\top \boldsymbol{\theta}^\star\leq u\}}du.
\end{aligned}$$
Denote the integral function by $D(\boldsymbol{x}; \boldsymbol{\delta})=\int_0^{\boldsymbol{x}^\top\boldsymbol{\delta}}\mathbf{1}_{\{0<1 -  \boldsymbol{x}^\top \boldsymbol{\theta}^\star\leq u\}}du.$ Then for all $\boldsymbol{\delta}\in\mathcal{C}(S,3)$, we have $|D(\boldsymbol{x}; \boldsymbol{\delta})|\leq |\boldsymbol{x}^\top\boldsymbol{\delta}|\leq c_n\|\boldsymbol{\delta}\|_1\leq 4c_n \|\boldsymbol{\delta}_S\|_1 \leq 4c_n \sqrt{s} \|\boldsymbol{\delta}\|_2,$ and 
$$ \mathbb{E}[D^2(\boldsymbol{x}; \boldsymbol{\delta})] \le  \mathbb{E} \left[ (\boldsymbol{x}^\top \boldsymbol{\delta})^2 \cdot \mathbf{1}_{\left\{ |1 - \boldsymbol{x}^\top \boldsymbol{\theta}^\star| \le |\boldsymbol{x}^\top \boldsymbol{\delta}| \right\}} \right], $$
by assumption (C1), we have that,  $f(\boldsymbol{x}^\top \boldsymbol{\theta}^\star|\boldsymbol{x})\leq C_{den}$ for $|1 - \boldsymbol{x}^\top \boldsymbol{\theta}^\star| \le |\boldsymbol{x}^\top \boldsymbol{\delta}|$, hence,
$$ \sigma_D^2 = \text{Var}(D) \le C_{den} \cdot \mathbb{E} \left[ (\boldsymbol{x}^\top \boldsymbol{\delta})^2 \cdot |\boldsymbol{x}^\top \boldsymbol{\delta}| \right]\le C c_n^3\|\boldsymbol{\delta}\|_1^3 \le C c_n^3s^{3/2}\|\boldsymbol{\delta}\|_2^3.$$
Using similar method in Lemma A.1 of \cite{wanglei1}, we obtain, with probability at least $1-n^{-C}$,
$$\begin{aligned} 
&\widehat{L}_1(\boldsymbol{\theta}^\star + \boldsymbol{\delta}) - \widehat{L}_1(\boldsymbol{\theta}^\star) - \boldsymbol{\delta}^\top \nabla \widehat{L}_1(\boldsymbol{\theta}^\star) - \mathbb{E}\left[\widehat{L}_1(\boldsymbol{\theta}^\star + \boldsymbol{\delta})\right] +\mathbb{E}\left[\widehat{L}_1(\boldsymbol{\theta}^\star)\right]\\
& \leq C \left( \frac{s^{3/4} (c_n \|\boldsymbol{\delta}\|_2)^{3/2} \sqrt{\log n}}{\sqrt{n}} + \frac{c_n \sqrt{s} \|\boldsymbol{\delta}\|_2 \log n}{n} \right). 
\end{aligned}$$

\noindent\textbf{Step 3.}  Assume $\|\widehat{\boldsymbol{\delta}}\|_2= t>0$, by the definition of $\widehat{\boldsymbol{\theta}}$, we have
$$ \widehat{\mathcal{L}}(\boldsymbol{\theta}^\star + \widehat{\boldsymbol{\delta}}) - \widehat{\mathcal{L}}(\boldsymbol{\theta}^\star) + \lambda\|\boldsymbol{\theta}^\star + \widehat{\boldsymbol{\delta}}\|_1 - \lambda\|\boldsymbol{\theta}^\star\|_1 \leq 0. $$We also know that, $-\|\boldsymbol{\theta}^\star + \widehat{\boldsymbol{\delta}}\|_1 + \|\boldsymbol{\theta}^\star\|_1 \leq \|\widehat{\boldsymbol{\delta}}_S\|_1 \leq \sqrt{s}\|\widehat{\boldsymbol{\delta}}_S\|_2 \leq \sqrt{s}t.$ Using the result from Step 2 and the lower bound for $\mathbb{E}\left[\widehat{L}_1(\boldsymbol{\theta}^\star + \boldsymbol{\delta})\right] - \mathbb{E}\left[\widehat{L}_1(\boldsymbol{\theta}^\star)\right]$, similar to Lemma 4 of \cite{quansparse}, we have
$$ \begin{aligned}
C(t^2 \land t) &\leq\mathbb{E}\left[\widehat{L}_1(\boldsymbol{\theta}^\star + \widehat{\boldsymbol{\delta}})\right] - \mathbb{E}\left[\widehat{L}_1(\boldsymbol{\theta}^\star)\right]\\
&\leq \widehat{L}_1(\boldsymbol{\theta}^\star + \widehat{\boldsymbol{\delta}}) - \widehat{L}_1(\boldsymbol{\theta}^\star) - \boldsymbol{\delta}^\top \nabla \widehat{L}_1(\boldsymbol{\theta}^\star)  + C \left( \frac{s^{3/4} (c_n t)^{3/2} \sqrt{\log n}}{\sqrt{n}} + \frac{c_n \sqrt{s} t \log n}{n} \right) \\
&= \widehat{\mathcal{L}}(\boldsymbol{\theta}^\star + \widehat{\boldsymbol{\delta}}) - \widehat{\mathcal{L}}(\boldsymbol{\theta}^\star) - \boldsymbol{\delta}^\top \nabla \widehat{\mathcal{L}}(\boldsymbol{\theta}^\star) + C \left( \frac{s^{3/4} (c_n t)^{3/2} \sqrt{\log n}}{\sqrt{n}} + \frac{c_n \sqrt{s} t \log n}{n} \right) \\
&\leq -\lambda(\|\boldsymbol{\theta}^\star + \widehat{\boldsymbol{\delta}}\|_1 - \|\boldsymbol{\theta}^\star\|_1 ) + \|\widehat{\boldsymbol{\delta}}\|_1\|\nabla \widehat{\mathcal{L}}(\boldsymbol{\theta}^\star)\|_\infty  + C \left( \frac{s^{3/4} (c_n t)^{3/2} \sqrt{\log n}}{\sqrt{n}} + \frac{c_n \sqrt{s} t \log n}{n} \right) \\
&\leq C\lambda\sqrt{s}t + C \left( \frac{s^{3/4} (c_n t)^{3/2} \sqrt{\log n}}{\sqrt{n}} + \frac{c_n \sqrt{s} t \log n}{n} \right). 
\end{aligned} $$
Some algebra shows that
$$ t \leq C \left( \lambda\sqrt{s} + \frac{c_n\sqrt{s}\log n}{n} + \frac{s^{3/2}c_n^2\log n}{n} \right) \leq C \left( \lambda\sqrt{s} + \frac{s^{3/2}c_n^2\log n}{n} \right), $$
which completes the proof of the first sub-conclusion in our argument.

\noindent\textbf{Step 4.} The argument proceeds along the same lines as the proof of Proposition~\ref{theorem3}, leading us to conclude that,
$$\|\nabla \widehat{\mathcal{L}}(\boldsymbol{\theta}^\star)\|_\infty\leq C\left( \sqrt{\frac{\log d}{N}} + \beta \sqrt{\frac{\log d}{n}}  + \frac{s \log d}{n^{3/4}}  \right),$$ with probability at least $1-(1-\alpha)md^{-C}$.
 On that event, with probability at least $1 - n^{-C}$,$$\|\widehat{\boldsymbol{\theta}} - \boldsymbol{\theta}^\star\|_ 2 \lesssim \sqrt{\frac{s\log d}{N}} + \beta \sqrt{\frac{s\log d}{n}} +  \frac{s^{3/2}\log d}{n^{3/4}}  .$$
\end{proof}

\subsection{Technical Lemmas}
We adapt the statement of the following lemma from \cite{wanglei1}.
\begin{lemma}\label{unibd}
	If $\left\|\boldsymbol{\theta}\right\| _{0}\leq Cs,  \left\|\boldsymbol{\theta}-\boldsymbol{\theta}^{\star} \right\| _{1}\leq R $, then as long as the derivate $\ell^{\prime}(\cdot)$ of loss function is Lipschitz and monotonic (or piecewise monotonic), with probability at least $1-d^{-C}$, we have
	
	$$
	\begin{aligned}
	& \left\|\frac{1}{n} \sum_i \boldsymbol{x}_i\left(\ell^{\prime}(y_i-\boldsymbol{x}_i^{\top} \boldsymbol{\theta})-\ell^{\prime}(y_i-\boldsymbol{x}_i^{\top} \boldsymbol{\theta}^{\star})\right) -\mathbb{E}\boldsymbol{x}\left(\ell^{\prime}(y-\boldsymbol{x}^{\top} \boldsymbol{\theta})-\ell^{\prime}(y-\boldsymbol{x}^{\top} \boldsymbol{\theta}^{\star})\right) \right\|_{\infty} \\
	& \quad \leq C\left(c_n R \sqrt{\frac{s \log d}{n}}+c_n \frac{s \log d}{n}\right) .
	\end{aligned}
	$$
	
\end{lemma}

\begin{proof}[Proof of Lemma~\ref{unibd}]
	Define the class of functions
	
	$$
	\mathcal{F}(\Omega)=\left\{x_j \mathbf{1}_{\left\{\max _j\left|x_j\right| \leq c_n\right\}}\left(\ell^{\prime}(y_i-\boldsymbol{x}_i^{\top} \boldsymbol{\theta})-\ell^{\prime}(y_i-\boldsymbol{x}_i^{\top} \boldsymbol{\theta}^{\star})\right): \boldsymbol{\theta} \in \Omega\right\},
	$$
	with envelope function $F(\boldsymbol{x}, y)=\left|x_j\right|$. 
	
	We take $$\Omega=\left\{\boldsymbol{\theta}:\|\boldsymbol{\theta}\|_0 \leq C s,\left\|\boldsymbol{\theta}-\boldsymbol{\theta}^{\star}\right\|_1 \leq R\right\}=\cup_{T \subset\{1, \ldots, d\},|T| \leq C s} \Omega(T)$$
	with $\Omega(T)=\{\boldsymbol{\theta}:$ support of $\boldsymbol{\theta} \subset$ $\left.T,\left\|\boldsymbol{\theta}-\boldsymbol{\theta}_0\right\|_1 \leq R\right\}$. By Lemma 2.6.15, Lemma 2.6.18(vi) and the proof of Lemma 2.6.18(viii) in \cite{van}, since $ \mathcal{F}(\Omega) $ is the class of multipliers and  compositions of finite dimensional functions and monotonic (or piecewise monotonic) functions, for each fixed $T$ with $|T| \leq C s$, $\mathcal{F}(\Omega(T))$ is a VC-subgraph with index bounded by $C s$. Then by Theorem 2.6.7 of \cite{van}, we have
	
	$$
	N\left(\epsilon, \mathcal{F}(\Omega(T)), L_2\left(P_n\right)\right) \leq\left(\frac{C\|F\|_{L_2\left(P_n\right)}}{\epsilon}\right)^{C s}.
	$$
	Since there are at most $\binom{d}{C s} \leq(e d /(C s))^{C s}$ different such $T$, we have
	
	$$
	N\left(\epsilon, \mathcal{F}(\Omega), L_2\left(P_n\right)\right) \leq\left(\frac{C\|F\|_{L_2\left(P_n\right)}}{\epsilon}\right)^{C s}\left(\frac{e d}{C s}\right)^{C s} \leq\left(\frac{C d\|F\|_{L_2\left(P_n\right)}}{\epsilon}\right)^{C s}
	$$

	With $\Omega$ set as above, let $\sigma_{j}^2=\sup _{f \in \mathcal{F}(\Omega)} P f^2$. Then by Theorem 3.12 of \cite{oracle}, with $\|F\|_{L_2(P)}$ obviously bounded by a constant, we have
	
	$$
	\mathbb{E}\left\|R_n\right\|_{\mathcal{F}(\Omega)} \leq C\left(\sigma_{j}\sqrt{\frac{s \log d}{n}}+\frac{c_n s \log d}{n}\right).
	$$
	where $\left\|R_n\right\|_{\mathcal{F}}=\sup _{f \in \mathcal{F}} n^{-1} \sum_{i=1}^n \varepsilon_i f\left(\boldsymbol{x}_i, y_i\right)$ with $\varepsilon_i$ being i.i.d. Rademacher random variables. 
	
	Using the symmetrization inequality which states that $E \| P_n-$ $P\left\|_{\mathcal{F}} \leq 2 E\right\| R_n \|_{\mathcal{F}}$, Talagrand's inequality (p.24 of \cite{oracle}) gives
	
	$$
	P\left(\left\|P_n-P\right\|_{\mathcal{F}} \geq C\left(\sigma_{j} \sqrt{\frac{s \log d}{n}}+c_n \frac{s \log d}{n}+\sqrt{\frac{\sigma_{j}^2 t}{n}}+c_n \frac{t}{n}\right)\right) \leq e^{-t},
	$$
	that is, with probability $1-d^{-C}$,
	
	$$
	\left\|P_n-P\right\|_{\mathcal{F}} \leq C\left(\sigma_{j} \sqrt{\frac{s \log d}{n}}+c_n \frac{s \log d}{n}\right),
	$$
hence,
	$$
	\begin{aligned}
	& \left\|\frac{1}{n} \sum_i \boldsymbol{x}_i\left(\ell^{\prime}(y_i-\boldsymbol{x}_i^{\top} \boldsymbol{\theta})-\ell^{\prime}(y_i-\boldsymbol{x}_i^{\top} \boldsymbol{\theta}^{\star})\right) -\mathbb{E}\boldsymbol{x}\left(\ell^{\prime}(y-\boldsymbol{x}^{\top} \boldsymbol{\theta})-\ell^{\prime}(y-\boldsymbol{x}^{\top} \boldsymbol{\theta}^{\star})\right) \right\|_{\infty} \\
	=& \max_{j\in [d]} \left|\frac{1}{n} \sum_i x_{i,j}\left(\ell^{\prime}(y_i-\boldsymbol{x}_i^{\top} \boldsymbol{\theta})-\ell^{\prime}(y_i-\boldsymbol{x}_i^{\top} \boldsymbol{\theta}^{\star})\right) -\mathbb{E}x_{\cdot,j}\left(\ell^{\prime}(y-\boldsymbol{x}^{\top} \boldsymbol{\theta})-\ell^{\prime}(y-\boldsymbol{x}^{\top} \boldsymbol{\theta}^{\star})\right) \right|\\
	 \leq &C\sup_{j\in[d]}\left(\sigma_{j} \sqrt{\frac{s \log d}{n}}+c_n \frac{s \log d}{n}\right).
	\end{aligned}
	$$
	
	It is easy to see that
	
	$$
	\sup_{j\in[d]}\sigma_{j}^2 \leq C \sup _{\boldsymbol{\theta} \in \Omega} \mathbb{E}\left[\left|\mathbf{1}_{\left\{\max _j\left|x_j\right| \leq c_n\right\}} \boldsymbol{x}_i^{\top}\left(\boldsymbol{\theta}-\boldsymbol{\theta}^{\star}\right)\right|^{2}\right] \leq C c_n^{2} R^{2}.
	$$

\end{proof}

We next give a deterministic bound for the trimmed-mean aggregator, showing how averaging over the honest workers mitigates the effect of Byzantine contamination by letting $\mu$ be the mean of the honest workers.
\begin{lemma}\label{trmean}
For any $\mu\in \mathbb{R}$. If $\beta>\alpha$, then
	$$\left|\operatorname{trmean}_{\beta}\left\{\widehat{g}^{(k)}_{j}: k \in [m]\right\}-\mu\right|\leq\frac{1}{(1-2 \beta) m}\left|\sum_{k \in \mathcal{M}}\left(\widehat{g}^{(k)}_{j}-\mu\right)\right|+\frac{2\beta+\alpha}{1-2 \beta}\max_{k \in  \mathcal{M}}\left|\widehat{g}^{(k)}_{j}-\mu\right|.$$
\end{lemma}

\begin{proof}[Proof of Lemma~\ref{trmean}]
	Denote $r_{j}:=\operatorname{argmax}_{k \in \mathcal{B} \cap \mathcal{U}_j(t)}\left|\widehat{g}^{(k)}_{j}-\mu\right|$. On the one hand, if $\widehat{g}^{(r_{j})}_{j}-\mu\geq 0$,
	since $\beta> \alpha$, we have $\exists i_{j} \in \mathcal{T}_j(t)\cap \mathcal{M}$, s.t. $ \widehat{g}^{(i_{j})}_{j}\geq \widehat{g}^{(r_{j})}_{j}$. Therefore $\left|\widehat{g}^{(i_{j})}_{j}-\mu\right|\geq \left|\widehat{g}^{(r_{j})}_{j}-\mu\right|$. On the other hand, if $\widehat{g}^{(r_{j})}_{j}-\mu< 0$, $\exists i_{j} \in \mathcal{T}_j(t)\cap \mathcal{M}$, s.t. $ \widehat{g}^{(i_{j})}_{j}\leq \widehat{g}^{(r_{j})}_{j}$. Therefore, $\left|\widehat{g}^{(i_{j})}_{j}-\mu\right|\geq \left|\widehat{g}^{(r_{j})}_{j}-\mu\right|$.
	\begin{equation}\label{eq11}
		\begin{aligned}
			&\left|\sum_{k \in \mathcal{B} \cap \mathcal{U}_j(t)}\left(\widehat{g}^{(k)}_{j}-\mu\right)\right|\leq\alpha m \max_{k \in \mathcal{B} \cap \mathcal{U}_j(t)}\left|\widehat{g}^{(k)}_{j}-\mu\right|\leq \alpha m  \max_{k \in \mathcal{T}_j(t)\cap \mathcal{M}}\left|\widehat{g}^{(k)}_{j}-\mu\right|\leq \alpha m  \max_{k \in  \mathcal{M}}\left|\widehat{g}^{(k)}_{j}-\mu\right|.
		\end{aligned}
	\end{equation}
	By (\ref{eq11}), we have
	$$
	\begin{aligned}
		&\left|\operatorname{trmean}_{\beta}\left\{\widehat{g}^{(k)}_{j}: k \in [m]\right\}-\mu\right|  =\left|\frac{1}{(1-2 \beta) m} \sum_{k \in \mathcal{U}_j(t)} \widehat{g}^{(k)}_{j}-\mu\right| \\
		=&\frac{1}{(1-2 \beta) m}\left|\sum_{k \in \mathcal{M}}\left(\widehat{g}^{(k)}_{j}-\mu\right)-\sum_{k \in \mathcal{M} \cap \mathcal{T}_j(t)}\left(\widehat{g}^{(k)}_{j}-\mu\right)+\sum_{k \in \mathcal{B} \cap \mathcal{U}_j(t)}\left(\widehat{g}^{(k)}_{j}-\mu\right)\right| \\
		\leq&\frac{1}{(1-2 \beta) m}\left(\left|\sum_{k \in \mathcal{M}}\left(\widehat{g}^{(k)}_{j}-\mu\right)\right|+\left|\sum_{k \in \mathcal{M} \cap \mathcal{T}_j(t)}\left(\widehat{g}^{(k)}_{j}-\mu\right)\right|+\left|\sum_{k \in \mathcal{B} \cap \mathcal{U}_j(t)}\left(\widehat{g}^{(k)}_{j}-\mu\right)\right|\right)\\
		\leq&\frac{1}{(1-2 \beta) m}\left(\left|\sum_{k \in \mathcal{M}}\left(\widehat{g}^{(k)}_{j}-\mu\right)\right|+2\beta m \max_{k \in \mathcal{M}}\left|\widehat{g}^{(k)}_{j}-\mu\right|+\alpha m  \max_{k \in  \mathcal{M}}\left|\widehat{g}^{(k)}_{j}-\mu\right|\right)\\ =&\frac{1}{(1-2 \beta) m}\left(\left|\sum_{k \in \mathcal{M}}\left(\widehat{g}^{(k)}_{j}-\mu\right)\right|+\left( 2\beta+\alpha\right)  m  \max_{k \in  \mathcal{M}}\left|\widehat{g}^{(k)}_{j}-\mu\right|\right).
	\end{aligned}
	$$

\end{proof}

\section{Numerical Experiments} \label{Secexper}
The synthetic experiments are organized around two guiding questions: (a) how close the distributed estimator can get to the oracle (global) in high dimensions; and (b) how the performance degrades as the Byzantine fraction $\alpha$ increases under common attack patterns.
The real-data experiments provide an empirical validation that these tradeoffs extend beyond the generative models considered in simulation.

Unless stated otherwise, we consider a distributed high-dimensional sparse linear model with total sample size $N=mn$, where $m$ is the number of worker machines and each worker holds $n$ observations. The covariates are generated independently from a multivariate normal distribution $N(\mathbf{0},\boldsymbol{\Sigma})$, where $\boldsymbol{\Sigma}$ is a Toeplitz covariance matrix with $(j,k)$-th entry $\Sigma_{jk}=0.5^{|j-k|}$. The response is generated according to
\[
Y_i=\boldsymbol{X}_i^\top \boldsymbol{\theta}^\star+\varepsilon_i,\qquad i=1,\ldots,N.
\]
For regression (pseudo-Huber and quantile regression), we fix
\[
\boldsymbol{\theta}^\star=(0.2,0.4,\ldots,2.0,0,\ldots,0)^\top \quad (s=10),
\]
and evaluate estimation error $\|\widehat{\boldsymbol{\theta}}-\boldsymbol{\theta}^{\star}\|_{2}$ together with support recovery, by the number of indices predicted as nonzero that are truly zero (FP), the number of truly nonzero  indices that are missed (FN), and their harmonic mean (F1). For classification, we record the test accuracy. Each configuration is repeated 100 times.

The attack mechanism is given by the following three approaches:
\begin{itemize}
	\item \textbf{Sign-flip attack:} Each Byzantine machine $k$ sends $\widehat{\boldsymbol{g}}^{(k)} = -\nabla \widehat{L}_k(\widehat{\boldsymbol{\theta}}_t)$, which is the most basic form of adversarial behavior that directly opposes the true gradient.
	\item \textbf{Random attack:} Each Byzantine machine $k$ sends $\widehat{\boldsymbol{g}}^{(k)} = \mathbf{r}$, where $\mathbf{r}$ is a random vector drawn from a Gaussian distribution, that is independent of the true gradients.
	\item \textbf{Zero-message attack:} Each Byzantine machine $k$ sends $\widehat{\boldsymbol{g}}^{(k)} = \mathbf{0}$, effectively withholding information and forcing the aggregator to rely solely on the honest machines' contributions.
\end{itemize}

We take the robost parameter $a=0.743$, and trimmed level $\beta=\alpha$ in the following experiments. As for the choice of penalty parameter $\lambda$, we use the BIC-type selection criterion according to \cite{Divide}. 

\subsection{Simulations: Pseudo-Huber Regression}
\label{sec:supp-exp-huber}

For the pseudo-Huber regression, local updates are computed via an I-LAMM type procedure \citep{lamm}. And we mainly compare 5 different estimation methods: 
\begin{itemize}
	\item \textbf{Global:} The oracle estimator that has access to all $N$ samples and computes the regularized M-estimator directly.
	\item \textbf{Local:} The master machine computes the regularized M-estimator using only its local data, also used for initialization for the distributed algorithms.
	\item \textbf{Trimean:} Estimator derived from our Algorithm using Option \uppercase\expandafter{\romannumeral1}.
	\item \textbf{Median:} Estimator derived from our Algorithm using Option \uppercase\expandafter{\romannumeral2}.
	\item \textbf{SLARD:} The proposed method that applies a robust aggregation scheme based on the \cite{Divide} using corresponding loss.
\end{itemize}

\begin{figure}[!htbp]
  \centering
  \includegraphics[width=1\linewidth]{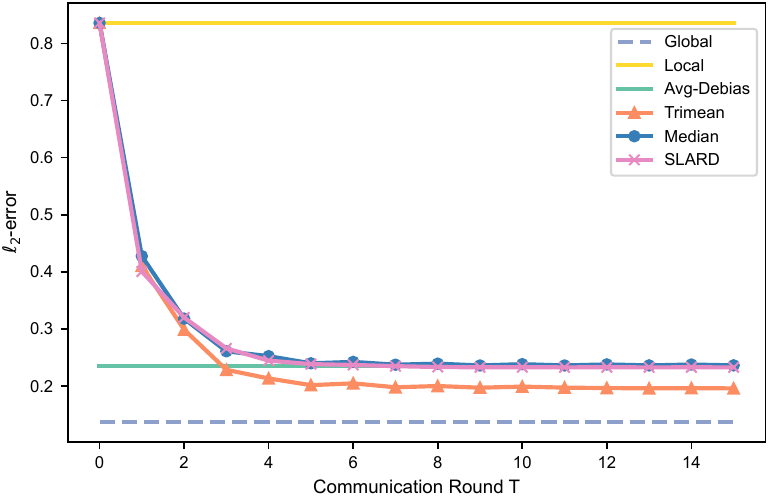}
  \caption{$\ell_2$-error versus communication round $T$ under pseudo-Huber loss with $t_3$ noise ($\alpha=0$). Setup: $(n,m,d)=(500,20,500)$. Corresponding curves from bottom to top: Global (dashed line, grey), Trimean (solid with triangles, orange), SLARD (solid with crosses, pink), Median (solid with circles, blue), Avg-Debias (solid line, green), Local (solid line, yellow). }
  \label{fig:huber-exp0-t3}
\end{figure}

\begin{table}[!htbp]
  \centering
  \caption{Pseudo-Huber Regression with Gaussian noise, $\alpha=0$, $(n,m,d)=(200,50,500)$.}
  \label{tab:huber_gaussian_alpha0}
  \begin{tabular}{@{}lcccccccc@{}}
  \toprule
  Estimator & Error & F1 & F1 Std & FP  & FP Std & FN & FN Std & Time(s) \\
  \midrule
  Global & 0.0976 & 0.82 & 0.11 & 4.68 & 3.08 & 0.00 & 0.00 & 0.89 \\
  Local & 0.6890 & 0.64 & 0.07 & 10.30 & 3.14 & 0.50 & 0.50 & 0.10 \\
  Avg-Debias & 0.2109 & 0.95 & 0.01 & 0.00 & 0.00 & 0.94 & 0.24 & 0.10 \\
  Trimean & 0.1643 & 1.00 & 0.01 & 0.07 & 0.29 & 0.00 & 0.00 & 0.17 \\
  Median & 0.2007 & 0.99 & 0.03 & 0.25 & 0.55 & 0.00 & 0.00 & 0.17 \\
  SLARD & 0.1892 & 0.99 & 0.02 & 0.21 & 0.45 & 0.00 & 0.00 & 0.27 \\
  \bottomrule
  \end{tabular}
\end{table}

\begin{table}[!htbp]
  \centering
  \caption{Pseudo-Huber Regression with Cauchy noise, $\alpha=0$, $(n,m,d)=(500,20,500)$.}
  \label{tab:huber_cauchy_alpha0}
  \begin{tabular}{@{}lcccccccc@{}}
  \toprule
   Estimator & Error & F1 & F1 Std & FP  & FP Std & FN & FN Std & Time(s) \\
  \midrule
  Global & 0.1770 & 0.92 & 0.07 & 1.96 & 2.47 & 0.00 & 0.00 & 0.76 \\
  Local & 4.8915 & 0.52 & 0.29 & 91.79 & 174.22 & 0.51 & 0.62 & 0.13 \\
  Avg-Debias & 0.4182 & 0.83 & 0.35 & 78.38 & 179.54 & 0.30 & 0.46 & 0.07 \\
  Trimean & 0.2216 & 0.97 & 0.04 & 0.66 & 0.96 & 0.00 & 0.00 & 0.21 \\
  Median & 0.2802 & 0.98 & 0.04 & 0.46 & 0.90 & 0.00 & 0.00 & 0.22 \\
  SLARD & (blowup) &  &  &  &  &  &  &  \\
  \bottomrule
  \end{tabular}
\end{table}

To illustrate the convergence behavior, we start with the distributed setting without Byzantine machines ($\alpha=0$) under normal, $t_3$ and Cauchy noise. And we consider \textbf{Avg-Debias}, each machine computes a local estimator and then applies a debiasing step before averaging across machines, as proposed in \cite{Communication1}, which is a common baseline in distributed estimation.

Figure~\ref{fig:huber-exp0-t3} tracks the $\ell_2$-error across communication rounds under $t_3$ noise, demonstrating how quickly each method approaches a stable neighborhood. It shows that iterative distributed estimators converge to near-oracle performance within 5--10 rounds. Hence, we record the time cost at 6 rounds for the iterative estimators as their running time. Tables~\ref{tab:huber_gaussian_alpha0}–\ref{tab:huber_cauchy_alpha0} summarize the zero-contamination baseline ($\alpha=0$) under Gaussian and Cauchy noise, respectively.
It can be observed that Byzantine-robust algorithms can achieve or even surpass the error level of Avg-Debias after only 3–4 iterations, which is consistent with our theoretical results.
This indicates that we do not sacrifice excessive accuracy for improving robustness.  And the gap between the distributed algorithm and the global estimator, coincides with the communication error, which is unavoidable in the distributed setting.

Meanwhile, it is evident that the SLARD algorithm imposes strict requirements on the quality of the initial estimate, and its iterations can easily become unstable. In the following simulations, unless stated otherwise,  we use $t_3$ noise for data generation.
\FloatBarrier

\begin{figure}[!htbp]
  \centering
  \includegraphics[width=\linewidth]{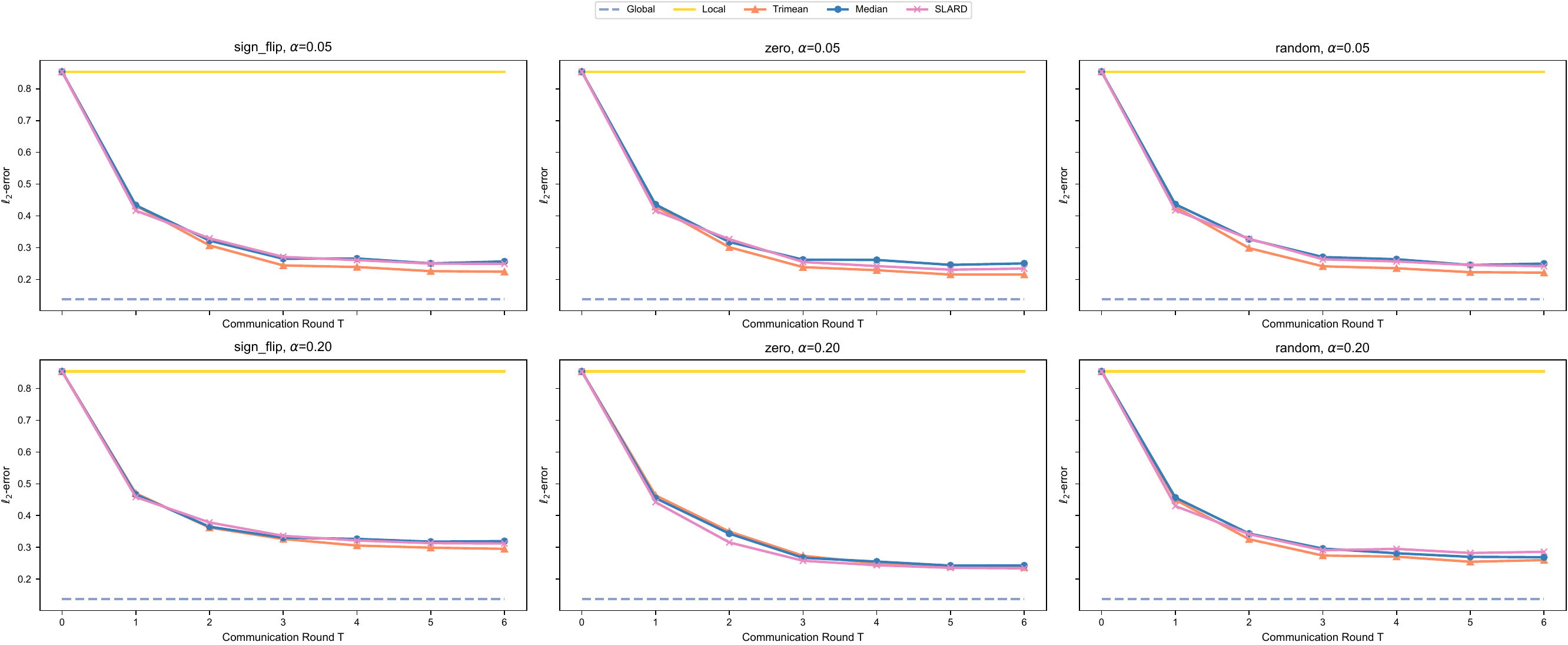}
  \caption{The $\ell_2$-error under pseudo-Huber loss over communication rounds, varying attack types and Byzantine ratios,with $t_3$ noise. $(n,m,d)=(200,50,500)$. }
  \label{fig:huber-exp1-convergence}
\end{figure}

\begin{figure}[!htbp]
	\centering
	\includegraphics[width=\linewidth]{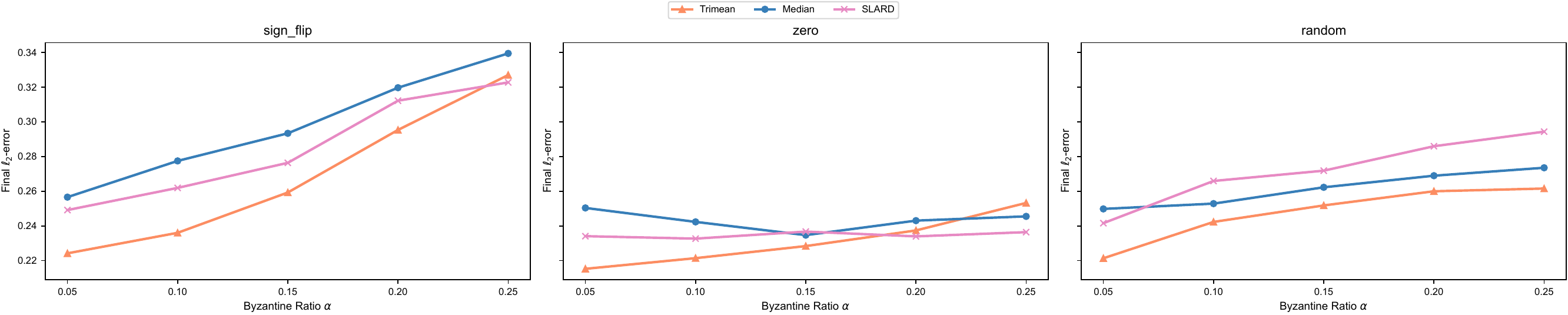}
	\caption{$\ell_2$-error under pseudo-Huber loss versus Byzantine ratio $\alpha$ (evaluated at the final round). Panels correspond to different attack types.}
	\label{fig:huber-exp1-alpha}
\end{figure}

We next fix $(n,m,d)=(200,50,500)$ and experiment for $\alpha \in \{0.00,\allowbreak 0.05,\allowbreak 0.10,\allowbreak 0.15,\allowbreak 0.20,\allowbreak 0.25\}$. Figure~\ref{fig:huber-exp1-convergence} still plots the error decay curves with the number of iterations. We observe that the introduction of Byzantine corruption does not cause significant performance degradation. Moreover, the convergence curve of our estimator exhibits a similar shape to that of SLARD,  while our estimators requires less computation time than SLARD. Figure~\ref{fig:huber-exp1-alpha} shows how the errors vary with the Byzantine ratio $\alpha$ for robust distributed estimators.
All three methods demonstrate robustness against different types of attacks.
Notably, the sign-flip attack has the largest impact on the estimation error, which matches its adversarial feature; in contrast, the zero attack causes the least degradation, which can be explained by the fact that it introduces little disruption to the direction of the optimization update .

\FloatBarrier

\begin{figure}[htbp]
  \centering
  \includegraphics[width=\linewidth]{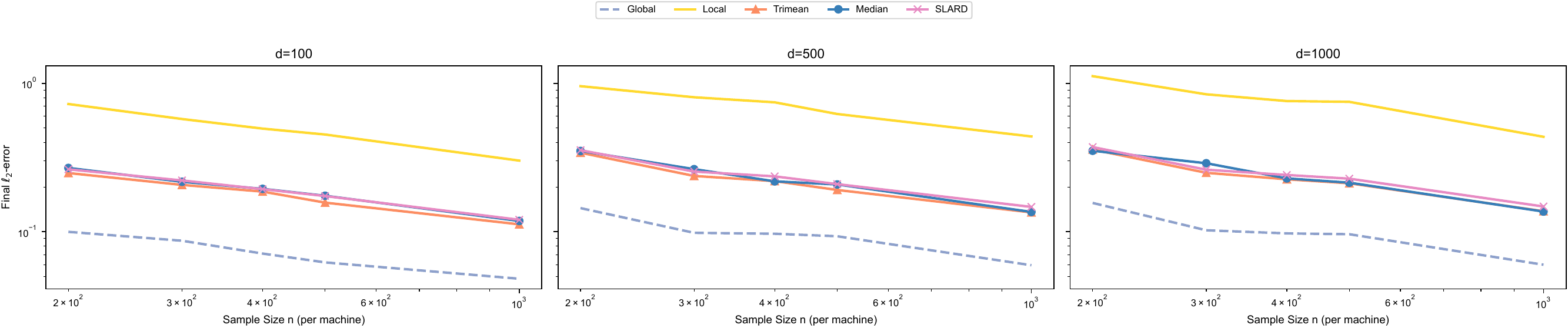}
  \caption{Final $\ell_2$-error under pseudo-Huber loss versus sample size $n$ (log-log scale). Columns correspond to dimensions $d\in\{100,500,1000\}$, with fixed $m=50$ and sign flip attack of the ratio $\alpha=0.2$. }
  \label{fig:huber-exp2}
\end{figure}

\begin{figure}[htbp]
  \centering
  \includegraphics[width=\linewidth]{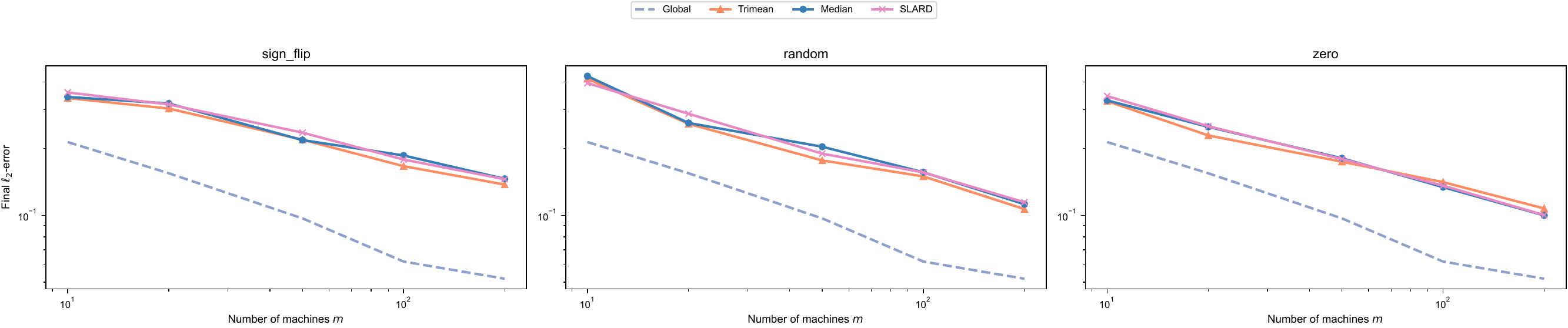}
  \caption{Final $\ell_2$-error under pseudo-Huber loss versus number of machines $m$ (log-log scale). Columns correspond to attack types, fixing $(n,d)=(400,500)$ and $\alpha=0.2$.}
  \label{fig:huber-exp3}
\end{figure}

To characterize scaling behavior, Figure~\ref{fig:huber-exp2} varies the per-machine sample size $n\in\{200,\allowbreak 300,\allowbreak 400,\allowbreak 500,\allowbreak 1000\}$, across multiple ambient dimensions $d$. Figure~\ref{fig:huber-exp3} then varies the number of machines $m\in\{10,\allowbreak 20,\allowbreak 50,\allowbreak 100,\allowbreak 200\}$ under fixed $(n,d)$, shedding light on how the distributed methods trade off additional machines against communication rounds and robustness. As we noted before, the distributed estimators behave well without limiting the number of machines according to the local sample size.

\subsection{Simulations: Quantile Regression}
\label{sec:supp-exp-quantile}

The quantile regression experiments largely mirror the pseudo-Huber design. To handle the challenges arising from the non-smoothness of the quantile loss, we employ the quantile huber-smoothing proposed in \cite{smooth}, to simplify our optimization procedure:
$$\rho_{\tau, h}(u) =
\begin{cases} 
u(\tau - 1) - \frac{(1-\tau)h}{2} & u \leq -h \\
\frac{u^2}{2h} + u(\tau - 0.5) + \frac{h}{8} & -h < u < h \\
u\tau - \frac{\tau h}{2} & u \geq h
\end{cases}$$
where $h$ is the bandwidth parameter.

\begin{figure}[!htbp]
  \centering
  \includegraphics[width=\linewidth]{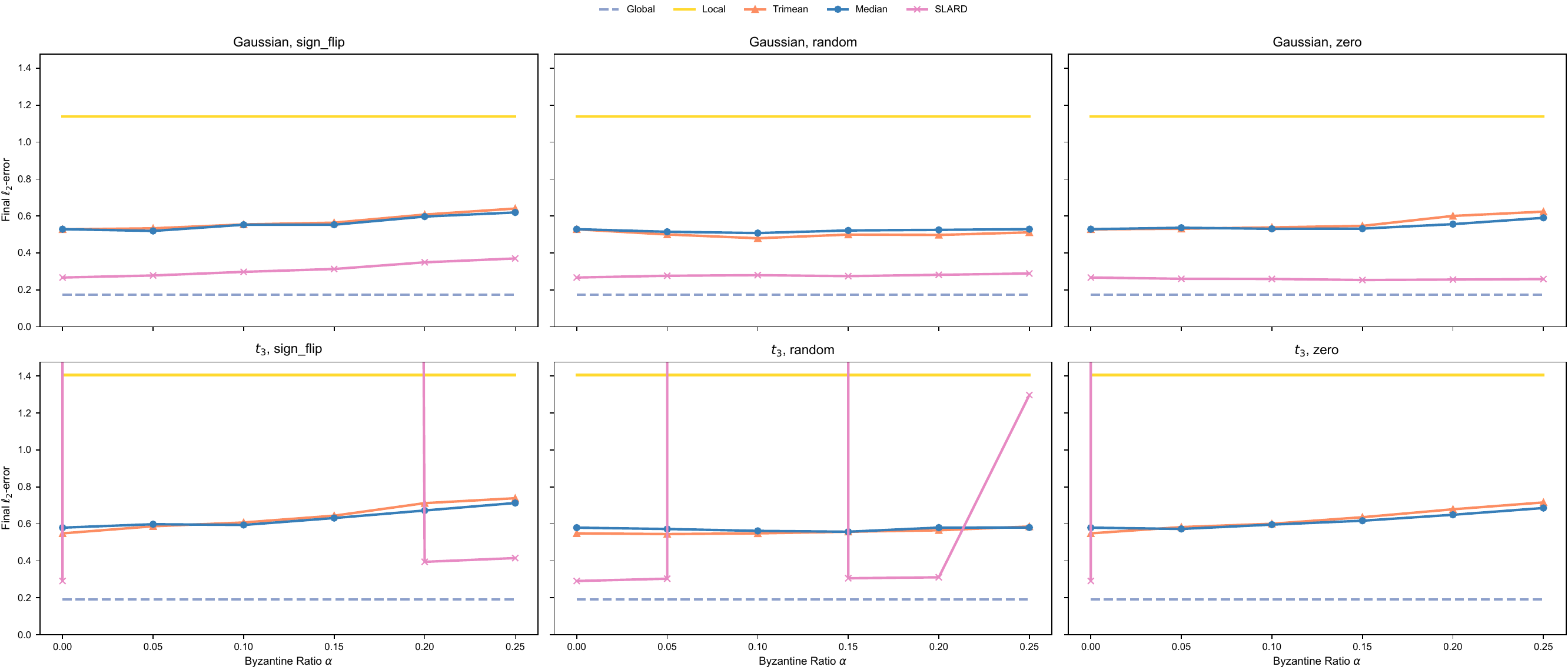}
  \caption{Final $\ell_2$-error under quantile loss versus Byzantine ratio $\alpha$. Rows correspond to noise distributions (Gaussian and $t_3$). Columns correspond to attack types. $(n,m,d)=(300,25,500)$ }
  \label{fig:qexp1_alpha}
\end{figure}

Figure~\ref{fig:qexp1_alpha} scans $\alpha$ under several noise/attack combinations as a robustness stress test: we focus on whether robust aggregators maintain bounded errors as $\alpha$ increases, and whether any method exhibits a sharp failure mode (reported as blowup with frequency over 100 independent trials). For the representative  setting, Gaussian noise with $\alpha=0.2$, Table~\ref{tab:qexp1_gaussian_alpha02} provides detailed results on support recovery behavior.

\begin{table}[!htbp]
  \centering
  \caption{Quantile Regression with Gaussian noise, $\alpha=0.2$, $(n,m,d)=(300,25,500)$.}
  \label{tab:qexp1_gaussian_alpha02}
  \begin{tabular}{llccccc}
    \toprule
    Attack & Estimator & Error (Std) & F1 (Std) & FP (Std) & FN (Std) & Time \\
    \midrule
    \multirow{5}{*}{sign flip}
      & Global  & 0.1746 (0.0148) & 0.94 (0.05) & 1.27 (1.28)  & 0.00 (0.00) & 0.51 \\
      & Local   & 1.1395 (0.1763) & 0.82 (0.07) & 2.90 (1.52)  & 1.09 (0.58) & 0.03 \\
      & Trimean & 0.6077 (0.0317) & 0.98 (0.02) & 0.04 (0.20)  & 0.26 (0.44) & 0.23 \\
      & Median  & 0.5969 (0.1117) & 0.98 (0.07) & 0.57 (4.98)  & 0.22 (0.44) & 0.27 \\
      & SLARD   & 0.3492 (0.0271) & 0.98 (0.03) & 0.37 (0.66)  & 0.00 (0.00) & 0.69 \\
    \midrule
    \multirow{3}{*}{random}
      & Trimean & 0.4973 (0.0939) & 0.98 (0.04) & 0.23 (0.66)  & 0.09 (0.29) & 0.28 \\
      & Median  & 0.5248 (0.1249) & 0.98 (0.04) & 0.31 (0.64)  & 0.13 (0.34) & 0.31 \\
      & SLARD   & 0.2816 (0.0289) & 0.96 (0.04) & 0.79 (0.89)  & 0.00 (0.00) & 0.69 \\
    \midrule
    \multirow{3}{*}{zero}
      & Trimean & 0.6003 (0.0584) & 0.98 (0.03) & 0.10 (0.39)  & 0.23 (0.42) & 0.16 \\
      & Median  & 0.5557 (0.0795) & 0.98 (0.09) & 2.51 (23.97) & 0.17 (0.38) & 0.15 \\
      & SLARD   & 0.2558 (0.0522) & 0.96 (0.05) & 0.82 (1.13)  & 0.00 (0.00) & 0.68 \\
    \bottomrule
  \end{tabular}
\end{table}

\FloatBarrier
\begingroup
\renewcommand{\topfraction}{0.95}
\renewcommand{\bottomfraction}{0.95}
\renewcommand{\textfraction}{0.03}
\renewcommand{\floatpagefraction}{0.8}

\begin{table}[!htbp]
  \centering
  \caption{Quantile regression with $t_{2.1}$ noise, $\alpha=0.2$, sign flip attack.}
  \label{tab:qexp2_error_f1}
  \begin{tabular}{lcccccc>{\raggedright\arraybackslash}p{2.2cm}}
    \toprule
    Setting & $m$ &   & Global & Local & Trimean & Median & SLARD \\
    \midrule
    \multirow{10}{*}{$n=200,d=500$}
      & \multirow{2}{*}{5}   & Error & 0.6290 & 1.8744 & 1.5679 & 1.5584 & blowup(40\%) \\
      &                      & F1    & 0.92   & 0.58   & 0.87   & 0.86   & -- \\
      & \multirow{2}{*}{10}  & Error & 0.3468 & 1.8560 & 1.5156 & 1.5177 & blowup(7\%) \\
      &                      & F1    & 0.88   & 0.60   & 0.88   & 0.88   & -- \\
      & \multirow{2}{*}{25}  & Error & 0.2083 & 1.8010 & 1.1767 & 1.3262 & blowup(2\%) \\
      &                      & F1    & 0.86   & 0.58   & 0.90   & 0.88   & -- \\
      & \multirow{2}{*}{50}  & Error & 0.1855 & 1.9220 & 1.0719 & 1.1076 & blowup(8\%) \\
      &                      & F1    & 0.95   & 0.60   & 0.91   & 0.89   & -- \\
      & \multirow{2}{*}{100} & Error & 0.1201 & 1.8594 & 0.9793 & 1.0175 & blowup(3\%) \\
      &                      & F1    & 0.95   & 0.60   & 0.91   & 0.91   & -- \\
    \midrule
    \multirow{10}{*}{$n=200,d=1000$}
      & \multirow{2}{*}{5}   & Error & 0.6273 & 2.2147 & 1.6186 & 1.5743 & blowup(12\%) \\
      &                      & F1    & 0.91   & 0.48   & 0.86   & 0.85   & -- \\
      & \multirow{2}{*}{10}  & Error & 0.3612 & 2.3831 & 1.6487 & 1.5755 & blowup(12\%) \\
      &                      & F1    & 0.84   & 0.48   & 0.87   & 0.88   & -- \\
      & \multirow{2}{*}{25}  & Error & 0.2228 & 2.5046 & 1.5112 & 1.4880 & blowup(26\%) \\
      &                      & F1    & 0.84   & 0.48   & 0.87   & 0.89   & -- \\
      & \multirow{2}{*}{50}  & Error & 0.1946 & 2.4302 & 1.3461 & 1.3941 & blowup(19\%) \\
      &                      & F1    & 0.98   & 0.47   & 0.88   & 0.89   & -- \\
      & \multirow{2}{*}{100} & Error & 0.1210 & 2.3810 & 1.2350 & 1.2463 & blowup(17\%) \\
      &                      & F1    & 0.96   & 0.46   & 0.86   & 0.88   & -- \\
    \midrule
    \multirow{10}{*}{$n=500,d=500$}
      & \multirow{2}{*}{5}   & Error & 0.3464 & 0.9124 & 0.6971 & 0.6806 & blowup(3\%) \\
      &                      & F1    & 0.93   & 0.74   & 0.96   & 0.95   & -- \\
      & \multirow{2}{*}{10}  & Error & 0.2100 & 0.9532 & 0.6636 & 0.6535 & 0.4000 \\
      &                      & F1    & 0.87   & 0.76   & 0.96   & 0.97   & 0.95 \\
      & \multirow{2}{*}{25}  & Error & 0.1379 & 0.8969 & 0.5025 & 0.5267 & 0.2758 \\
      &                      & F1    & 0.89   & 0.73   & 0.98   & 0.98   & 0.93 \\
      & \multirow{2}{*}{50}  & Error & 0.1160 & 0.9426 & 0.4571 & 0.4363 & 0.2356 \\
      &                      & F1    & 0.97   & 0.76   & 0.99   & 0.99   & 0.99 \\
      & \multirow{2}{*}{100} & Error & 0.0736 & 0.8856 & 0.4275 & 0.4139 & 0.1624 \\
      &                      & F1    & 0.96   & 0.73   & 1.00   & 0.99   & 0.98 \\
    \midrule
    \multirow{10}{*}{$n=500,d=1000$}
      & \multirow{2}{*}{5}   & Error & 0.3479 & 1.2804 & 0.7069 & 0.7017 & 1.1195 \\
      &                      & F1    & 0.91   & 0.84   & 0.95   & 0.93   & 0.91 \\
      & \multirow{2}{*}{10}  & Error & 0.2244 & 1.2227 & 0.7029 & 0.6800 & blowup(2\%) \\
      &                      & F1    & 0.84   & 0.82   & 0.96   & 0.97   & -- \\
      & \multirow{2}{*}{25}  & Error & 0.1542 & 1.2303 & 0.5913 & 0.5944 & 0.3347 \\
      &                      & F1    & 0.90   & 0.84   & 0.97   & 0.97   & 0.95 \\
      & \multirow{2}{*}{50}  & Error & 0.1190 & 1.2691 & 0.5243 & 0.5241 & 0.2456 \\
      &                      & F1    & 0.98   & 0.84   & 0.98   & 0.98   & 0.99 \\
      & \multirow{2}{*}{100} & Error & 0.0735 & 1.2303 & 0.4763 & 0.4694 & 0.1612 \\
      &                      & F1    & 0.95   & 0.83   & 0.98   & 0.99   & 0.97 \\
    \bottomrule
  \end{tabular}
\end{table}

Now we use noise of $t_{2.1}$ distribution, which has only second moments finite, with a sign flip attack of ratio $\alpha=0.2$.  As we can see from Table~\ref{tab:qexp2_error_f1},  on different $(n,m,d)$ settings, the proposed method still performs well in terms of estimation error and support recovery, while SLARD fails to converge on small $n$ and small $m$ settings-- the former is due to the fact that the initial estimator is rough, and the latter is because the robust aggregation information is limited and not stable when the number of machines is small.

Finally, we use simulation results for $\tau = 0.25, 0.5$ and $0.75$ in Table~\ref{tab:qexp3_tau} to show that, the robustness holds universally across different quantile settings.

\begin{table}[!htbp]
  \centering
  \caption{Quantile regression for $(n,m,d)=(400,25,500)$.}
  \label{tab:qexp3_tau}
  \begin{tabular}{lcccccc}
    \toprule
    $\tau$ &   & Global & Local & Trimean & Median & SLARD \\
    \midrule
    \multirow{4}{*}{0.25}
      & Error & 0.1955 & 1.4471 & 0.7124 & 0.6915 & 0.4050 \\
      & F1    & 0.96 & 0.82 & 0.96 & 0.96 & 0.99 \\
      & FP    & 0.77   & 2.10   & 0.03   & 0.09   & 0.26   \\
      & FN    & 0.00   & 1.62   & 0.74   & 0.60   & 0.00   \\
    \midrule
    \multirow{4}{*}{0.5}
      & Error & 0.1806 & 1.3336 & 0.6650 & 0.6594 & 0.3725 \\
      & F1    & 0.95 & 0.86 & 0.96 & 0.96 & 0.98 \\
      & FP    & 1.25   & 1.55   & 0.12   & 0.15   & 0.56   \\
      & FN    & 0.00   & 1.45   & 0.68   & 0.54   & 0.00   \\
    \midrule
    \multirow{4}{*}{0.75}
      & Error & 0.1960 & 1.3962 & 0.7143 & 0.6793 & 0.4004 \\
      & F1    & 0.96 & 0.84 & 0.96 & 0.96 & 0.98 \\
      & FP    & 0.90   & 1.76   & 0.04   & 0.09   & 0.37   \\
      & FN    & 0.00   & 1.51   & 0.70   & 0.61   & 0.00   \\
    \bottomrule
  \end{tabular}
\end{table}

\FloatBarrier
\endgroup

\subsection{Simulations: Sparse SVM}
\label{sec:supp-exp-svm}

We next turn to sparse SVM classification to demonstrate that the same communication-robustness tradeoffs extend beyond regression losses. As comparison, we calculate the non-robust estimator \textbf{Mean} proposed by \cite{svm}, simply averaging the transmitted gradients at each round before optimization. We consider two models for sparse SVM: 

\begin{itemize}
	\item \textbf{Model 1:} We use the correlated Gaussian design from regression: $\boldsymbol{X}_i\sim \mathcal{N}(\boldsymbol{0},\boldsymbol{\Sigma})$ with Toeplitz covariance $\Sigma_{jk}=0.5^{|j-k|}$. The signal $\boldsymbol{\theta}^\star\in\mathbb{R}^d$ is $s$-sparse, with its nonzero entries drawn i.i.d.\ from $\mathrm{Unif}(0.6,1.2)$. Given $\boldsymbol{X}_i$, labels follow a probit model $P(Y_i=1\mid \boldsymbol{X}_i)=\Phi(\boldsymbol{X}_i^\top\boldsymbol{\theta}^\star)$ with $Y_i\in\{-1,1\}$. 
	\item \textbf{Model 2:} We consider a balanced Gaussian-mixture model with $Y\in\{-1,1\}$ and $P(Y=1)=P(Y=-1)=0.5$. Conditional on the class label,
$\boldsymbol{X}\mid (Y=1)\sim \mathcal{N}(\boldsymbol{\mu},\boldsymbol{\Sigma}),\quad \boldsymbol{X}\mid (Y=-1)\sim \mathcal{N}(-\boldsymbol{\mu},\boldsymbol{\Sigma}).$
We set $s=5$ and
$\boldsymbol{\mu}=(0.1, 0.2, 0.3, 0.4, 0.5, 0,\dots,0)^\top \in \mathbb{R}^d.$
The covariance matrix $\boldsymbol{\Sigma}=(\sigma_{ij})$ is constructed as follows: $\sigma_{ii}=1$ for all $i=1,\dots,d$; for the first $s$ coordinates, $\sigma_{ij}=\rho=-0.2$ whenever $1\le i\neq j\le s$; and all other off-diagonal entries are zero.

\end{itemize}
The $N$ training samples are evenly split across $m$ workers, and for Model 1 we evaluate accuracy on an independent test set of size $N_{\text{test}}=N$. For Model 2, the Bayes optimal classification rule is given by
$\operatorname{sgn}\bigl(2.67X_1+2.83X_2+3X_3+3.17X_4+3.33X_5\bigr)$,
with the corresponding Bayes error approximately $6.3\%$.
A weak negative correlation structure is only introduced among the first $s$ informative coordinates, while all other coordinates are mutually independent.

\begin{figure}[htbp]
  \centering
  \includegraphics[width=\linewidth]{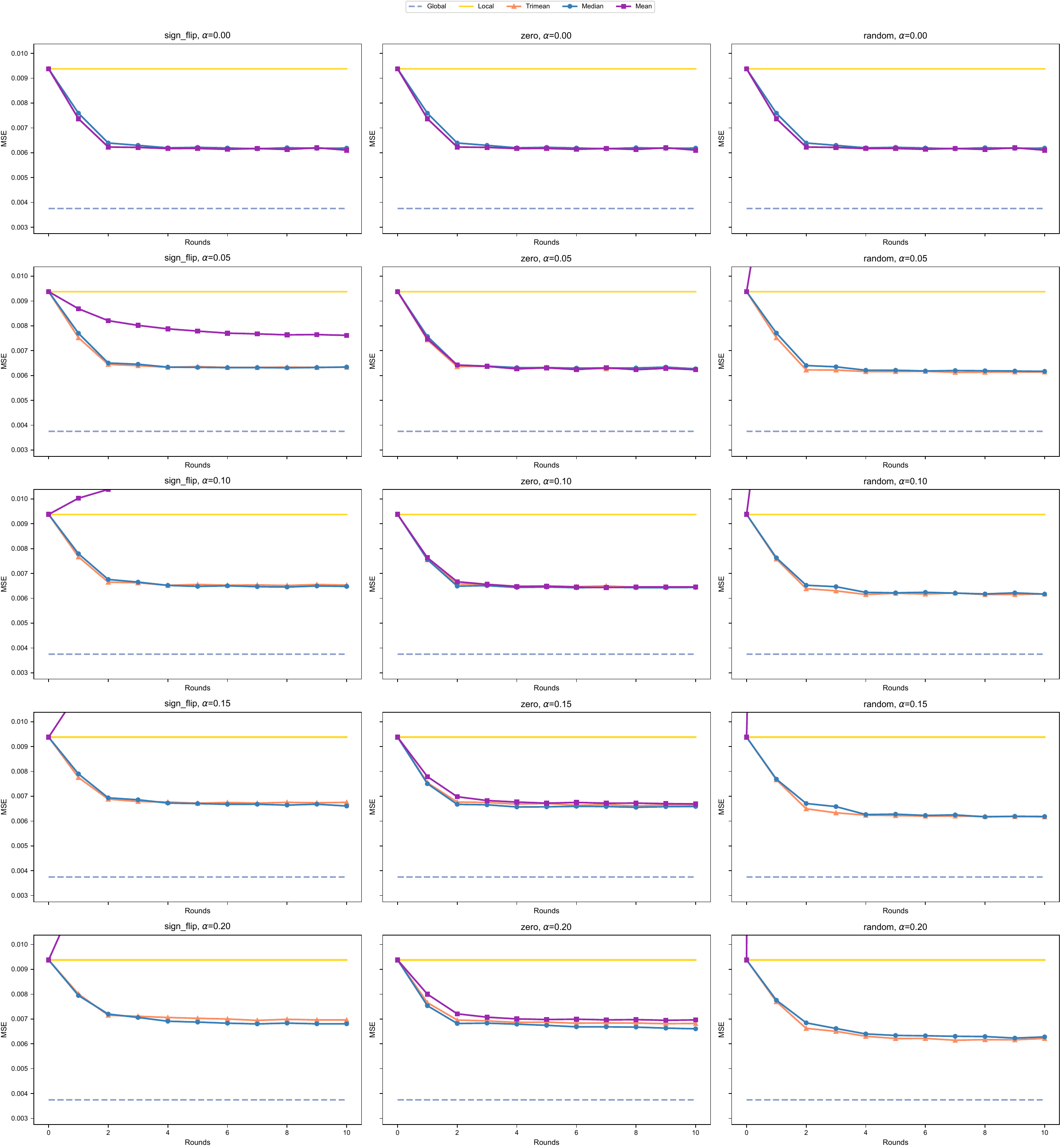}
  \caption{MSE of sparse SVM versus rounds in Model 1. $(n,m,d)=(400,20,500)$. Panels correspond to $(\alpha,\mathrm{attack})$ configurations. }
  \label{fig:svm-sexp1-conv}
\end{figure}

Figure~\ref{fig:svm-sexp1-conv} illustrates how MSE evolves with communication rounds under different $(\alpha,\text{attack})$ setups.
The robust aggregators converge to a stable error level within 5–10 rounds, and the final error is not significantly affected by $\alpha$ or attack type. In contrast, the non-robust Mean estimator exhibits blow-up in most cases, especially as $\alpha$ increases.
Since simply observing MSE may mask decision-boundary effects, as shown in Table~\ref{tab:classification_accuracies}, the distributed estimator without a robust aggregator achieves low test accuracy under different values of $\alpha$ and attack types. In contrast, the Median and Trimean estimators attain test accuracy close to that of the distributed estimator under no attack across all settings. These results convincingly demonstrate the importance of robust operations in distributed classification.

\begin{table}[htbp]
  \centering
  \caption{Classification accuracies (\%) of the sparse SVM estimators for Model 1.}
  \label{tab:classification_accuracies}
  \resizebox{\textwidth}{!}{
  \begin{tabular}{lcc@{\hspace{8pt}}ccc@{\hspace{8pt}}ccc}
    \toprule
    Attack & \multicolumn{2}{c}{$\alpha=0$ (none)} & \multicolumn{3}{c}{$\alpha=0.1$} & \multicolumn{3}{c}{$\alpha=0.2$} \\
    \cmidrule(lr){2-3}\cmidrule(lr){4-6}\cmidrule(lr){7-9}
    & Mean=Trimean & Median & Mean & Median & Trimean & Mean & Median & Trimean \\
    \midrule
    sign flip & \textbf{92.94}$\pm$0.48 & \textbf{92.88}$\pm$0.52 & \textbf{88.95}$\pm$1.87 & \textbf{92.86}$\pm$0.50 & \textbf{92.90}$\pm$0.50 & \textbf{50.11}$\pm$0.57 & \textbf{92.79}$\pm$0.52 & \textbf{92.86}$\pm$0.51 \\
    zero      & \textbf{92.94}$\pm$0.48 & \textbf{92.88}$\pm$0.52 & \textbf{92.95}$\pm$0.49 & \textbf{92.83}$\pm$0.51 & \textbf{92.92}$\pm$0.49 & \textbf{92.90}$\pm$0.49 & \textbf{92.84}$\pm$0.51 & \textbf{92.86}$\pm$0.49 \\
    random    & \textbf{92.94}$\pm$0.48 & \textbf{92.88}$\pm$0.52 & \textbf{52.92}$\pm$6.13 & \textbf{92.84}$\pm$0.50 & \textbf{92.88}$\pm$0.51 & \textbf{50.85}$\pm$2.20 & \textbf{92.83}$\pm$0.54 & \textbf{92.85}$\pm$0.54 \\
    \bottomrule
  \end{tabular}}
\end{table}

\begin{figure}[htbp]
  \centering
  \includegraphics[width=\linewidth]{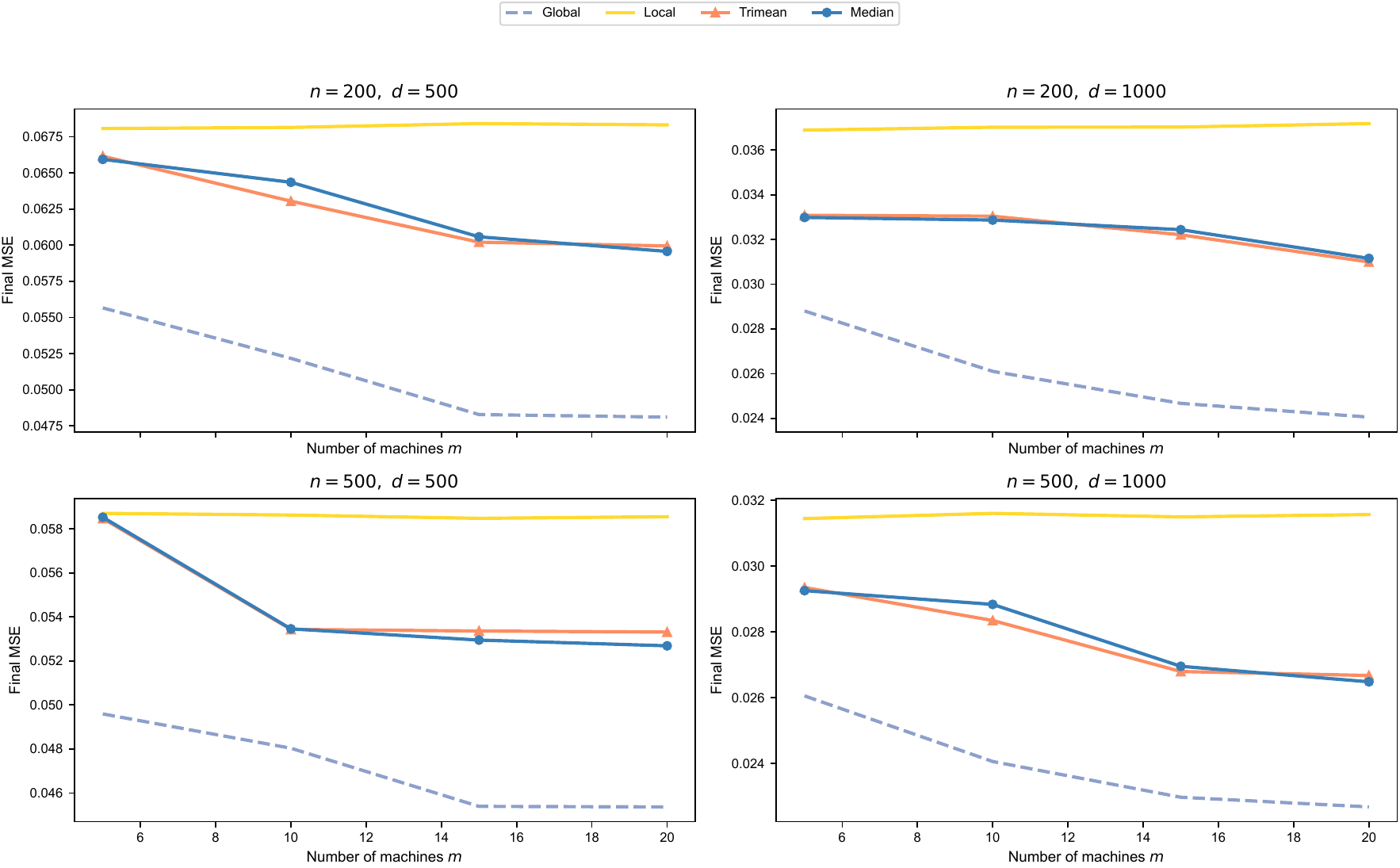}
  \caption{Final MSE of sparse SVM versus number of machines $m$ in Model 2.}
  \label{fig:svm-sexp2-mse}
\end{figure}

\begin{table}[p]
  \centering
  \caption{Model 2 results with $\alpha=0.2$ under the sign-flip attack.}
  \label{tab:model2_results}
  \begin{tabular}{lccccccc}
    \toprule
    Setting & $m$ &   & Global & Local & Trimean & Median & Mean \\
    \midrule
    \multirow{8}{*}{$n=200,d=500$}
      & \multirow{2}{*}{5}  & MSE & 0.0557 & 0.0681 & 0.0662 & 0.0659 & 0.0907 \\
      &                     & F1  & 0.80 & 0.63 & 0.98 & 0.94 & 0.00 \\
      & \multirow{2}{*}{10} & MSE & 0.0522 & 0.0681 & 0.0630 & 0.0644 & 0.0906 \\
      &                     & F1  & 0.95 & 0.61 & 0.96 & 0.97 & 0.00 \\
      & \multirow{2}{*}{15} & MSE & 0.0483 & 0.0684 & 0.0602 & 0.0606 & 0.0906 \\
      &                     & F1  & 0.65 & 0.65 & 0.97 & 0.95 & 0.00 \\
      & \multirow{2}{*}{20} & MSE & 0.0481 & 0.0683 & 0.0599 & 0.0596 & 0.0906 \\
      &                     & F1  & 0.86 & 0.62 & 0.99 & 0.97 & 0.00 \\
    \midrule
    \multirow{8}{*}{$n=500,d=500$}
      & \multirow{2}{*}{5}  & MSE & 0.0496 & 0.0587 & 0.0585 & 0.0585 & 0.0906 \\
      &                     & F1  & 0.68 & 0.63 & 0.98 & 0.98 & 0.00 \\
      & \multirow{2}{*}{10} & MSE & 0.0480 & 0.0586 & 0.0534 & 0.0535 & 0.0906 \\
      &                     & F1  & 0.95 & 0.63 & 0.98 & 0.95 & 0.00 \\
      & \multirow{2}{*}{15} & MSE & 0.0454 & 0.0585 & 0.0534 & 0.0529 & 0.0906 \\
      &                     & F1  & 0.63 & 0.64 & 1.00 & 0.98 & 0.00 \\
      & \multirow{2}{*}{20} & MSE & 0.0454 & 0.0586 & 0.0533 & 0.0527 & 0.0906 \\
      &                     & F1  & 0.85 & 0.63 & 1.00 & 1.00 & 0.00 \\
    \midrule
    \multirow{8}{*}{$n=200,d=1000$}
      & \multirow{2}{*}{5}  & MSE & 0.0288 & 0.0369 & 0.0331 & 0.0330 & 0.0453 \\
      &                     & F1  & 0.94 & 0.84 & 0.98 & 0.90 & 0.00 \\
      & \multirow{2}{*}{10} & MSE & 0.0261 & 0.0370 & 0.0330 & 0.0329 & 0.0453 \\
      &                     & F1  & 0.90 & 0.84 & 0.99 & 1.00 & 0.00 \\
      & \multirow{2}{*}{15} & MSE & 0.0247 & 0.0370 & 0.0322 & 0.0324 & 0.0453 \\
      &                     & F1  & 0.62 & 0.83 & 0.98 & 0.98 & 0.00 \\
      & \multirow{2}{*}{20} & MSE & 0.0241 & 0.0372 & 0.0310 & 0.0312 & 0.0453 \\
      &                     & F1  & 0.74 & 0.86 & 0.98 & 0.97 & 0.00 \\
    \midrule
    \multirow{8}{*}{$n=500,d=1000$}
      & \multirow{2}{*}{5}  & MSE & 0.0261 & 0.0314 & 0.0294 & 0.0293 & 0.0453 \\
      &                     & F1  & 0.97 & 0.87 & 0.99 & 0.97 & 0.00 \\
      & \multirow{2}{*}{10} & MSE & 0.0241 & 0.0316 & 0.0283 & 0.0288 & 0.0453 \\
      &                     & F1  & 0.90 & 0.88 & 0.99 & 0.99 & 0.00 \\
      & \multirow{2}{*}{15} & MSE & 0.0230 & 0.0315 & 0.0268 & 0.0270 & 0.0453 \\
      &                     & F1  & 0.56 & 0.87 & 1.00 & 0.97 & 0.00 \\
      & \multirow{2}{*}{20} & MSE & 0.0227 & 0.0316 & 0.0267 & 0.0265 & 0.0453 \\
      &                     & F1  & 0.75 & 0.88 & 1.00 & 1.00 & 0.00 \\
    \bottomrule
  \end{tabular}
\end{table}

We experiment with Model 2 under the sign flip attack with $\alpha=0.2$, with different $(n,d)$ settings. Table~\ref{tab:model2_results} shows the distance between the estimated and true parameter, as well as the support recovery performance. The robust aggregators achieve stable estimation and support recovery across different $m$ values, while the non-robust Mean estimator fails to recover any informative features. 
Figure~\ref{fig:svm-sexp2-mse} further plots the final error with respect to $m$ across multiple $(n,d)$ settings under $\alpha=0.2$ sign flip attack, showing that adding more machines stably improves distributed estimation performance.
Those observations confirm our theoretical results.

\subsection{Real Data}
\label{sec:supp-exp-real}

The real-data experiments are included to verify that the synthetic conclusions are not driven solely by idealized distributional assumptions. For regression, we use the pseudo regression to analyze the the \texttt{Ames Housing dataset},\footnote{Dataset: \url{https://jse.amstat.org/v19n3/decock/AmesHousing.txt}; documentation: \url{https://jse.amstat.org/v19n3/decock/DataDocumentation.txt}.} which contains 2,930 residential sales records and 82 variables. We randomly partitioned the dataset into training and testing sets with a 80/20 split.  We distinguish numeric, ordinal, and nominal observations; ordinal variables are encoded through specific level mappings, nominal variables are one-hot encoded, and all features are standardized. After removing the identifier columns, we model $\log(1+\texttt{SalePrice})$ as the response. The intercept is appended without regularization. In the distributed experiment, the training data are evenly partitioned across $m=20$ workers, and we evaluate three Byzantine attacks for $\alpha\in\{0.1,0.2\}$ with noise variance of $5.0$.

\begin{figure}[htbp]
  \centering
  \includegraphics[width=\linewidth]{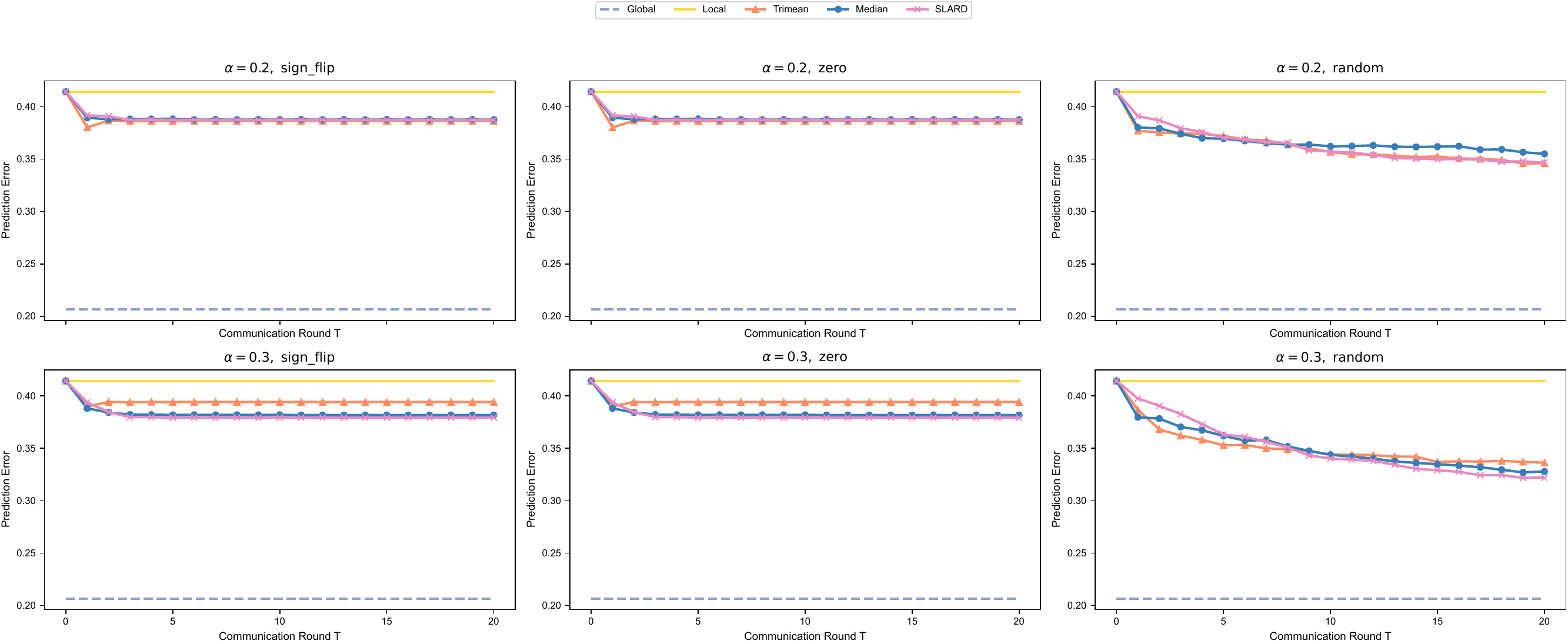}
  \caption{Prediction error for \texttt{Ames Housing dataset}. The total  training sample size $N$ is 2344, and the dimension of features $d$ is 244. }
  \label{fig:real-ames}
\end{figure}

Figure~\ref{fig:real-ames} summarizes the prediction error curves on the  \texttt{Ames Housing dataset}.
The global estimator serves as a natural benchmark.
Under Byzantine perturbations, all the robust methods exhibit the stable performance, while in most of the settings, median-based methods converges slightly faster than the Trimean method.

For classification of sparse SVM, we use two real binary datasets, \texttt{a9a} and \texttt{madelon}, from LIBSVM.\footnote{\url{https://www.csie.ntu.edu.tw/~cjlin/libsvmtools/datasets/binary.html}} We adopt the official training/test splits, whose raw dimensions are $(32561,123)$ and $(16281,123)$ for \texttt{a9a}, and $(2000,500)$ and $(600,500)$ for \texttt{madelon}. In all distributed experiments, the training samples are randomly shuffled and evenly split across $m=20$ workers. Byzantine settings are same as that of the regression experiment. 

\begin{figure}[!htbp]
  \centering
  \includegraphics[width=\linewidth]{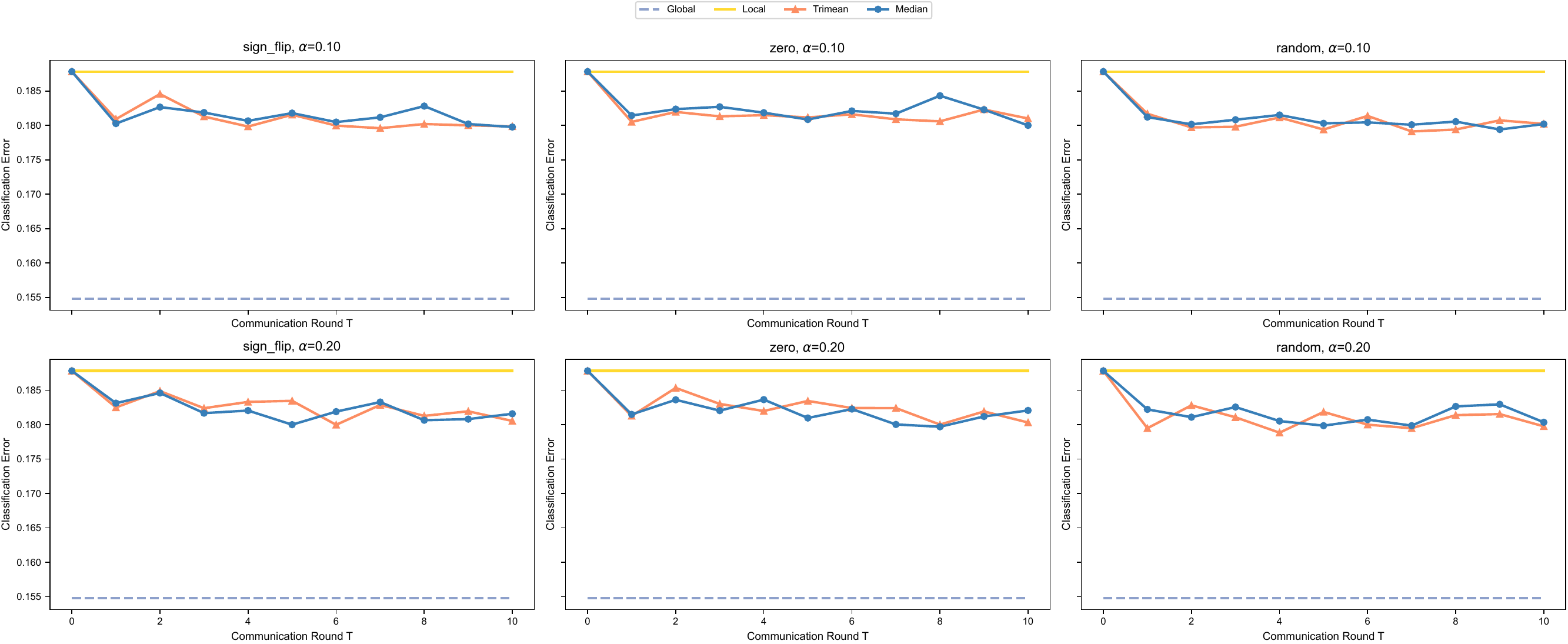}\par\vspace{6pt}
  \includegraphics[width=\linewidth]{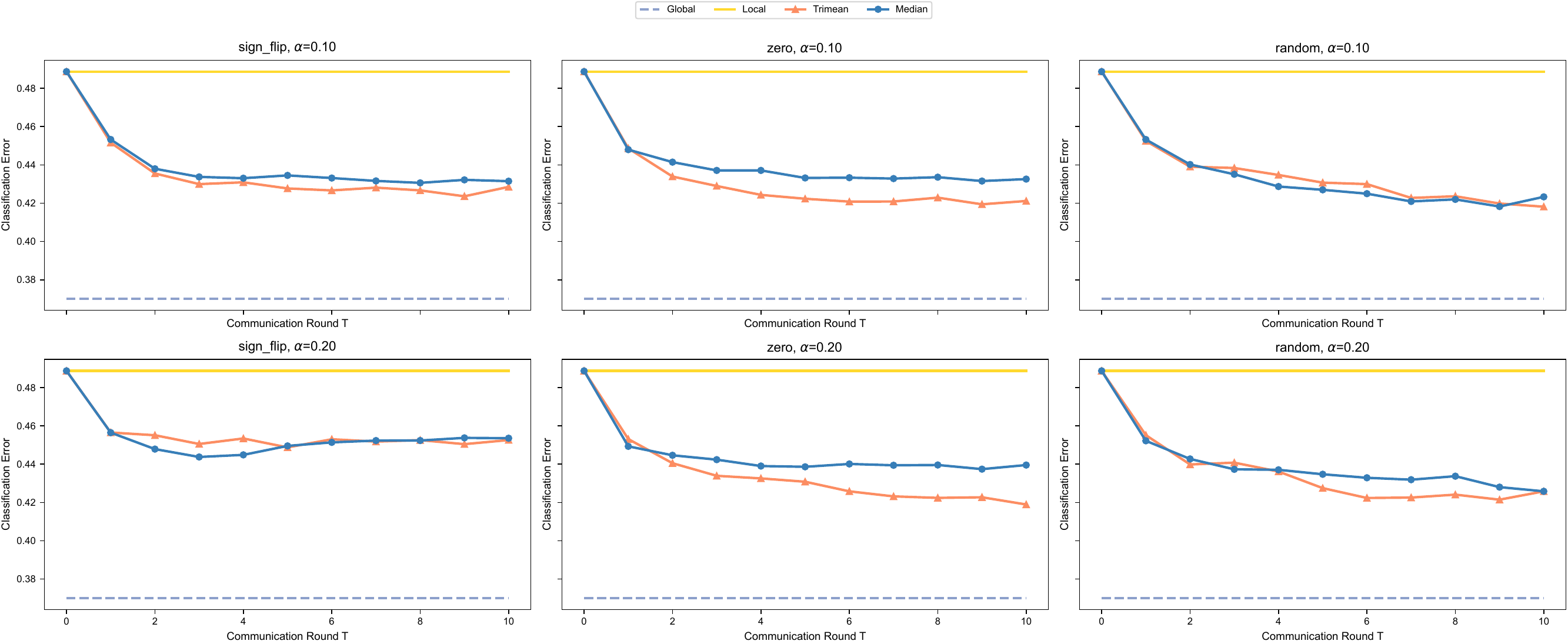}
  \caption{Classification error for real binary datasets. Top: \texttt{a9a}. Bottom: \texttt{madelon}.}
  \label{fig:real-svm}
\end{figure}

Figure~\ref{fig:real-svm} shows the resulting classification error trajectories.  On \texttt{a9a}, the convergence curves of the robust distributed estimators are rather flat, indicating that there is little room for further error reduction on this dataset. Nevertheless, the trimean and median methods still achieve consistent improvements over the local method. On \texttt{madelon}, the robust distributed estimators converge to a stable level within 5–10 rounds, and the final error is not significantly affected by $\alpha$ or attack type.
Even when the local initial estimator performs poorly, our distributed algorithm can reduce the gap relative to the global estimator.

\ifx\papermode\combinedmode
\else
\bibliographystyle{cas-model2-names}
\bibliography{references}
\fi
\suppenddocument

\fi

\printcredits

\bibliographystyle{cas-model2-names}

\bibliography{references}



\end{document}